\documentclass{article}

\usepackage{microtype}
\usepackage{graphicx}
\usepackage{subfigure}
\usepackage{booktabs} 

\usepackage{hyperref}

\usepackage[accepted]{icml2021}

\usepackage[utf8]{inputenc} 
\usepackage[T1]{fontenc}    
\usepackage{hyperref}       
\usepackage{url}            
\usepackage{booktabs}       
\usepackage{amsfonts}       
\usepackage{nicefrac}       
\usepackage{microtype}      
\usepackage{amsmath,amssymb}
\usepackage{amsthm}
\usepackage{times}
\usepackage{multicol} 
\usepackage{bm}
\usepackage{dsfont}
\usepackage{framed}
\usepackage{graphicx} 
\usepackage{subfigure} 
\usepackage{subfigmat}

\usepackage{xr}

\usepackage{color}
\usepackage{array}
\usepackage{comment}

\usepackage[symbol]{footmisc}

\theoremstyle{theorem}
\newtheorem{theorem}{Theorem}[section]
\newtheorem{lemma}[theorem]{Lemma}
\newtheorem{proposition}[theorem]{Proposition}
\newtheorem{corollary}[theorem]{Corollary}

\theoremstyle{definition}

\newtheorem{assumption}{Assumption}
\theoremstyle{remark}
\newtheorem*{remark}{Remark}

\newtheorem*{pr: lem: descent}{Proof of Lemma \ref{lem: descent}}
\newtheorem*{pr: lem: local_grad_var_nonconvex}{Proof of Lemma \ref{lem: local_grad_var_nonconvex}}
\newtheorem*{pr: prop: local_grad_bound_nonconvex}{Proof of Proposition \ref{prop: local_grad_bound_nonconvex}}
\newtheorem*{pr: lem: global_grad_var_bound_nonconvex}{Proof of Lemma \ref{lem: global_grad_var_bound_nonconvex}}
\newtheorem*{pr: thm: nonconvex}{Proof of Theorem \ref{thm: nonconvex}}

\icmltitlerunning{Bias-Variance Reduced Local SGD for Less Heterogeneous Federated  Learning}

\begin{document}

\twocolumn[
\icmltitle{Bias-Variance Reduced Local SGD for Less Heterogeneous Federated  Learning}



\icmlsetsymbol{equal}{*}

\begin{icmlauthorlist}
\icmlauthor{Tomoya Murata}{msi,tokyo}
\icmlauthor{Taiji Suzuki}{tokyo,riken}
\end{icmlauthorlist}

\icmlaffiliation{msi}{NTT DATA Mathematical Systems Inc., Tokyo, Japan}
\icmlaffiliation{tokyo}{Graduate School of Information Science and Technology, The University of Tokyo, Tokyo, Japan}
\icmlaffiliation{riken}{Center for Advanced Intelligence Project, RIKEN, Tokyo, Japan}

\icmlcorrespondingauthor{Tomoya Murata}{murata@msi.co.jp}
\icmlcorrespondingauthor{Taiji Suzuki}{taiji@mist.i.u-tokyo.ac.jp}

\icmlkeywords{Machine Learning, ICML}

\vskip 0.3in
]



\printAffiliationsAndNotice{}  

\begin{abstract}
Recently, local SGD has got much attention and been extensively studied in the distributed learning community to overcome the communication bottleneck problem. 
However, the superiority of local SGD to minibatch SGD only holds in quite limited situations. 
In this paper, we study a new local algorithm called Bias-Variance Reduced Local SGD (BVR-L-SGD) for nonconvex distributed optimization. 
Algorithmically, our proposed bias and variance reduced local gradient estimator fully utilizes small second-order heterogeneity of local objectives and suggests randomly picking up one of the local models instead of taking the  average of them when workers are synchronized. Theoretically, under small heterogeneity of local objectives, we show that BVR-L-SGD achieves better communication complexity than both the previous non-local and local methods under mild conditions, and particularly BVR-L-SGD is the first method that breaks the barrier of communication complexity $\Theta(1/\varepsilon)$
for general nonconvex smooth objectives when the heterogeneity is small and the local computation budget is large. 
Numerical results are given to verify the theoretical findings and give empirical evidence of the superiority of our method.
\end{abstract}

\section{Introduction}\label{sec: intro}
Nowadays, optimization problems arising in machine learning are often large and require huge computational time. \emph{Distributed learning} is one of the attractive approaches to reduce the computational time by utilizing parallel computing. In classical distributed learning, each worker has the whole dataset used in optimization or a random subset of the whole dataset which is not explicitly exchanged. 
In recent \emph{federated learning}, introduced by  \citet{konevcny2015federated, shokri2015privacy, mcmahan2017communication},
we build a global model across multiple devices or servers without explicitly exchanging their own local datasets, and \emph{local datasets can be heterogeneous}, i.e., each local dataset may be generated from a different distribution. There are various federated learning scenarios (e.g., personalization, preservation of the privacy of local information, robustness to attacks and failures, guarantees of fairness) and refer to the extensive survey \cite{kairouz2019advances} for these topics. 
\par
One of the most naive and widely used approaches to distributed learning is \emph{minibatch Stochastic Gradient Descent} (SGD) \cite{dekel2012optimal}, which is also called as Federated Averaging (FedAvg) \cite{mcmahan2017communication}. Each worker computes minibatch stochastic gradient of the own local objective and then their average is used to update the global model. Also, more computationally efficient methods including minibatch Stochastic Variance Reduced Gradient (SVRG) \cite{johnson2013accelerating, allen2016variance, reddi2016stochastic} and its variant \cite{lei2017non}, minibatch StochAstic Recursive grAdient algoritHm (SARAH) \cite{nguyen2017sarah, nguyen2019finite} and its variants \cite{fang2018spider, zhou2018stochastic} are applicable to the problem. Particularly, SARAH achieves the optimal total computational complexity in nonconvex optimization. 
\par
Unfortunately, \emph{minibatch methods suffer from their communication cost} because of the necessity to communicate local gradients for every single global update. One of the possible solutions is using a large batch to compute local gradients \cite{goyal2017accurate}, but the communication complexity, that is the necessary number of communication rounds to optimize, is theoretically never smaller than the one of deterministic GD and thus communication cost is still problematic. 
\par
To overcome the communication bottleneck problem, \emph{local methods} have got much attention due to its empirical effectiveness \cite{konevcny2015federated, lin2018don}. In local SGD (also called Parallel Restart SGD), each worker independently updates the local model based on his own local dataset, and periodically communicates and averages the local models. \emph{Many papers \cite{stich2018local, yu2019parallel, haddadpour2019convergence, haddadpour2019local, koloskova2020unified, khaled2020tighter} have stated the superiority of local SGD to minibatch SGD, but these results are based on unfair comparisons} and hence not satisfactory. Concretely, they have compared the two algorithms with the same local minibatch size $b$, which means that local SGD with $K$ local updates requires $K$ times larger number of local computations per communication round than minibatch SGD.\footnotemark\footnotetext{Practically, it is said that a larger minibatch size sometimes causes bad generalization ability in deep learning and thus comparing minibatch SGD and local SGD with a common local minibatch size is meaningful in some sense. But at least from a theoretical point of view, this comparison is questionable. } \emph{If we fix the number of single stochastic gradient computations per communication round to $\mathcal B := Kb$ for each worker, their results indicates that the communication complexity of local SGD with $K$ local updates and $b$ local minibatch size are never better than the one of minibatch SGD with $\mathcal B$ local minibatch size for any $\mathcal B$}.
This point is quite important, but many papers have overlooked it. \par

Recently, \citet{woodworth2020local, woodworth2020minibatch} have shown that, for the first time, \emph{theoretical superiority of local SGD to minibatch SGD under fair comparison} for convex optimization \emph{when the heterogeneity of local objectives is small}. On the other hand, their derived lower bound of local SGD also suggests the limitation of local SGD. Specifically, they have shown that if the first-order heterogeneity of local objectives\footnotemark\footnotetext{First-order heterogeneity $\zeta_1$ is defined as follows: $\|\nabla f_p(x) - \nabla f_{p'}(x)\| \leq \zeta_1, \forall x \in \mathbb{R}^d, \forall p, p' \in [P]$. } is greater than $\sqrt{\varepsilon}$, where $\varepsilon$ is the desired optimization accuracy, the communication complexity of local SGD is even worse than the one of minibatch SGD. In other words, \emph{the quite small heterogeneity of local objectives is essential for the superiority of local SGD to minibatch SGD}, which is a clear limitation of local SGD. SCAFFOLD \cite{karimireddy2020scaffold} is a new local algorithm based on the idea of reducing their called client-drift, which uses a similar formulation to the variance reduction technique. However, the communication complexity is the same as minibatch SGD for general nonconvex objectives and it requires small heterogeneity and \emph{quadraticity} of local objectives to surpass minibatch SGD, which is also quite limited. Inexact DANE \cite{reddi2016aide} is another variant of local methods that uses a general local subsolver that returns an approximate minimizer of the regularized local objective. Again, the superiority to non-local methods has been only shown for quadratic convex objectives.
\par
In summary, both in classical distributed learning and recent federated learning, naive minibatch (i.e., non-local) methods often suffer from their communication cost. Several local methods surpass non-local ones in terms of communication complexity. However, \emph{the necessary conditions for the superiority of the previous local algorithms to non-local ones are quite limited} (i.e., extremely small heterogeneity or quadraticity of local objectives). A natural question is that: {\bf{\emph{is there a local algorithm which surpasses non-local (and existing local) ones in terms of communication complexity with a fixed local computation budget under more relaxed conditions?}}}     

\begin{table*}[t]
    \centering
    \scalebox{1.0}{
    \begin{tabular}{c c c}\hline
        Algorithm & Communication Complexity & Extra Assumptions \\ \hline
        \begin{tabular}{c}
        Minibatch SGD
        \end{tabular}& $\frac{1}{\varepsilon} + \frac{1}{\mathcal B P\varepsilon^2}$ & None \\ \hline 
        \begin{tabular}{c}
        Minibatch SARAH \cite{nguyen2019finite}
        \end{tabular} & $\frac{1}{\varepsilon} + \frac{\sqrt{n\wedge\frac{1}{\varepsilon}}}{\mathcal B P\varepsilon}$ & None \\ \hline
        \begin{tabular}{c}Local SGD \cite{yu2019parallel}\end{tabular}& $\frac{1}{\mathcal B\varepsilon} + \frac{1}{\mathcal B P\varepsilon^2}+ \frac{1}{\varepsilon^\frac{3}{2}}$ & $G$ gradient boundedness \\ \hline
        Local SGD \cite{khaled2020tighter}\footnotemark & $\frac{1 + \sigma_\mathrm{dif}^4}{\mathcal B P \varepsilon^2} +  \frac{\sigma_\mathrm{dif}^2 \mathcal BP}{\varepsilon}$ & \begin{tabular}{c}convexity \\ $\mathcal B \leq \frac{1 + \sigma_\mathrm{dif}^2}{P\varepsilon} $ or $\sigma_\mathrm{dif}^2 \geq \varepsilon$ \end{tabular} \\ \hline
        \begin{tabular}{c}Local SGD \cite{woodworth2020minibatch}\end{tabular}  & $\frac{1}{\mathcal B\varepsilon} + \frac{1}{\mathcal B P\varepsilon^2} + \frac{1}{\sqrt{\mathcal B}\varepsilon^\frac{3}{2}} +  \frac{\zeta_1}{\varepsilon^\frac{3}{2}}$ & \begin{tabular}{c}convexity, \\$1$st-order $\zeta_1$ heterogeneity\end{tabular} \\ \hline
        \begin{tabular}{c}Local SGD \cite{woodworth2020minibatch} \\(Lower bound)
        \end{tabular}
        & $\frac{1}{\mathcal B\varepsilon^\frac{3}{2}} + \frac{1}{\mathcal B P\varepsilon^2} + \left(\frac{1}{\varepsilon}\wedge  \frac{\zeta_1}{\varepsilon^\frac{3}{2}}\right)$ & \begin{tabular}{c}convexity, \\$1$st-order $\zeta_1$ heterogeneity\end{tabular} \\ \hline        
        \begin{tabular}{c}SCAFFOLD \cite{karimireddy2020scaffold}\end{tabular}& $\frac{1}{\varepsilon} + \frac{1}{\mathcal B P\varepsilon^2}$ & None\\ \hline
        \begin{tabular}{c}SCAFFOLD \cite{karimireddy2020scaffold}\end{tabular}& $\frac{1}{\mathcal B \varepsilon} + \frac{1}{\mathcal B P\varepsilon^2} + \frac{\zeta}{\varepsilon}$ &  \begin{tabular}{c}quadraticity, \\$2$nd-order $\zeta$ heterogeneity\end{tabular}\\ \hline
        \begin{tabular}{c}BVR-L-SGD ({\color{red}this paper})\end{tabular} & $\frac{1}{\sqrt{\mathcal B}\varepsilon} + \frac{\sqrt{n\wedge\frac{1}{\varepsilon}}}{\mathcal B P\varepsilon} +  \frac{\zeta}{\varepsilon} $ & \begin{tabular}{c}$2$nd-order $\zeta$ heterogeneity \end{tabular} \\\hline
\end{tabular}
    }
    \caption{Comparison of the order of the necessary number of communication rounds to satisfy $\mathbb{E}\|f(x_{\mathrm{out}})\|^2 \leq \varepsilon$ (or $f(x_\mathrm{out}) - f(x_*) \leq \varepsilon$ for convex $f$). "Extra Assumptions'' indicates the necessary assumptions to derive the results other than Assumptions \ref{assump: local_loss_smoothness}, \ref{assump: optimal_sol} and \ref{assump: bounded_loss_gradient_variance} in Section \ref{sec: problem_setting}. $\varepsilon$ is the desired optimization accuracy. $\mathcal B$ is the local computation budget that is the allowed number of single stochastic gradient computations per communication round for each worker. $P$ is the number of workers. $n$ is the total number of samples and is possibly $\infty$ in online (i.e., expected risk  minimization) settings. The smoothness of local objectives $L$, the variance of a single stochastic gradient $\sigma^2$ and gradient boundedness $G$ are regarded as $\Theta(1)$ for ease of presentation. Note that in this notation, \emph{second-order heterogeneity $\zeta$ always satisfies $\zeta \leq \Theta(L) = \Theta(1)$}. $\sigma_\mathrm{dif}^2$ is the squared local gradient norm at an optimum (for the precise definition, see \cite{khaled2020tighter}). }
    \label{tab: theoretical_comparison}
\end{table*}

\footnotetext{Note that from the extra assumption, the communication complexity is always lower bounded by $\frac{1}{\varepsilon} + \frac{1}{\mathcal B P\varepsilon^2}$ for any  $\mathcal B$ even if $\sigma_\mathrm{diff} = 0$ (i.e., we are in overparamterized regimes). Thus, \emph{the communication complexity is never better than the one of minibatch SGD}.} 
\subsection*{Main Contributions}
We propose a new local algorithm called Bias-Variance Reduced Local SGD (BVR-L-SGD) for nonconvex distributed learning. The main features of our method are as below. \par 
{\bf{Algorithmic Features.}} The algorithm is based on our {\bf{\emph{bias and variance reduced gradient estimator}}} that simultaneously reduces the bias caused by local updates and the variance caused by stochastization based on the idea of SARAH like variance reduction technique. Importantly, \emph{to fully utilize the second-order heterogeneity of local objectives, {\bf{\emph{a randomly picked local model}}} is used as a synchronized global model} instead of taking the average of them, which is typical in the previous local methods. \par

{\bf{Theoretical Features.}} We analyse BVR-L-SGD for general nonconvex smooth objectives \emph{{\bf{\emph{under second-order heterogeneity assumption}}}}, which \emph{interpolates the heterogeneity of local objectives between the identical case and the extremely non-IID case}, and plays a critical role in our nonconvex analysis. The comparison of the communication complexities of our method with the most relevant existing results is given in Table \ref{tab: theoretical_comparison}. \emph{{\bf{\emph{The communication complexity of BVR-L-SGD has a better dependence on $\varepsilon$ than minibatch SGD, local SGD and SCAFFOLD}}}}. \emph{When $\mathcal B P \gg \sqrt{n \wedge 1/\varepsilon}$ and the second-order heterogeneity $\zeta$ of local objectives is small relative to the smoothness $L$}, \emph{{\bf{\emph{BVR-L-SGD strictly  surpasses minibatch SARAH}}}}. Furthermore, \emph{{\bf{\emph{BVR-L-SGD is the first method that breaks the barrier of communication complexity $1/\varepsilon$}}}} \emph{when local computation budget $\mathcal B \to \infty$, for general smooth nonconvex objectives with small heterogeneity $\zeta$}. Importantly, \emph{{\bf{\emph{even when the heterogeneity is high}}}, {\bf{\emph{the communication complexity of our method is never worse than the ones of the existing methods}}}} since the second-order heterogeneity $\zeta$ is bounded by two times the smoothness $L$ of local objectives\footnotemark\footnotetext{For the details, see Assumption \ref{assump: heterogeneous} in Section \ref{sec: problem_setting}}. \par

As a result, BVR-L-SGD is a novel and promising communication efficient method for nonconvex optimization both in classical distributed learning (i.e., local data distributions are nearly identical) and recent federated learning (i.e., local data distributions can be highly  heterogeneous).  \par 

{\bf{Other Related Work.}} Several recent papers have also studied local algorithms combined with variance reduction technique \cite{sharma2019parallel, das2020faster, karimireddy2020mime}. \citet{sharma2019parallel} have considered Parallel Restart SPIDER (PR-SPIDER), that is a local variant of SPIDER \cite{fang2018spider} and shown that the proposed algorithm achieves the optimal total computational complexity and the communication complexity of $1/\varepsilon$ for noncovnex smooth objectives. However, these rates essentially match the ones of non-local SARAH and no advantage of localization has been shown. Also, \citet{das2020faster} have considered a SPIDER like local algorithm called FedGLOMO but the derived communication complexity is only $1/\varepsilon^{3/2}$ in general and the rate is even worse than minibatch SARAH.  \citet{karimireddy2020mime} have proposed MIME, which is essentially a combination of local SGD and SVRG-like variance reduction technique. They have shown that MIME achieves the communication complexity of $1/(\mathcal B \varepsilon) + 1/(P^{3/4}\varepsilon^{3/2}) + \zeta/\varepsilon$ for $\zeta$ second-order heterogeneous nonconvex smooth objectives. Importantly, the second term of the rate of BVR-L-SGD has better dependencies on $P$ and $\mathcal B$ than the one of MIME. Particularly, BVR-L-SGD achieves $\zeta/\varepsilon$ when $\mathcal B \to \infty$ but MIME does not possess this property.

\section{Problem Definition and Assumptions}\label{sec: problem_setting}
In this section, we first introduce the notations used in this paper. Then, the problem setting considered in this paper is illustrated and theoretical assumptions are given. \par

{\bf{Notation.}}
$\| \cdot \|$ denotes the Euclidean $L_2$ norm $\| \cdot \|_2$: $\|x\| = \sqrt{\sum_{i}x_i^2}$ for vector $x$. 
For a matrix $X$, $\|X\|$ denotes the induced norm by the Euclidean $L_2$ norm. For a natural number $m$, $[m]$ denotes the set $\{1, 2, \ldots, m\}$.
For a set $A$, $\# A$ means the number of elements, which is possibly $\infty$. For any number $a, b$, $a \vee b$ and $a \wedge b$  denote $\mathrm{max}\{a, b\}$ and $\mathrm{min}\{a, b\}$ respectively. We denote the uniform distribution over $A$ by $\mathrm{Unif}(A)$. 
\subsection{Problem Setting}We want to minimize nonconvex smooth objective 
$$f(x) := \frac{1}{P}\sum_{p=1}^P f_p(x), \text{ where } f_p(x) := \mathbb{E}_{z \sim  D_p}[\ell(x, z)]$$
for $x \in \mathbb{R}^d$, where $D_p$ is the data distribution associated with worker $p$. Although we consider both offline  (i.e., $\# \mathrm{supp}(D_p) < \infty$ for every $p \in [P]$) and online (i.e., $\# \mathrm{supp}(D_p) = \infty$ for some $p \in [P]$) settings, it is assumed for offline settings that each local dataset has an equal number of samples, i.e.,  $\# \mathrm{supp}(D_p) = \# \mathrm{supp}(D_{p'})$ for every $p, p' \in [P]$ just for simplicity. We assume that each worker $p$ can only access the own data distribution $D_p$ without communication. Aggregation (e.g., summation) of all the worker's $d$-dimensional parameters or broadcast of a $d$-dimensional parameter from one worker to the other workers can be realized by single communication. In typical situations, single communication is more time-consuming than single stochastic gradient computation. Let $\mathcal C$ denotes the single communication cost and $\mathcal G$ does the single stochastic gradient computation. Using these notations, we assume $\mathcal C \geq \mathcal G$. Since we expect that a larger number of available stochastic gradients in a communication round leads to faster optimization, we can increase the number of stochastic gradient computations unless the total gradient computational time exceeds $\mathcal C$. This motivates the concept of {\bf{\emph{local computation budget $\mathcal B$}}} ($\leq \mathcal C / \mathcal G$): given a communication and computational environment, it is assumed that \emph{each worker can only compute at most $\mathcal B$ single stochastic gradients per communication round on average}. Then, we compare the communication complexity, that is the total number of communication rounds of a distributed optimization algorithm to achieve the desired optimization accuracy. From the definition, given a communication and computational environment, the communication complexity with a fixed local computation budget $\mathcal B := \mathcal C/\mathcal G$ captures the best achievable total execution time of an algorithm. Generally, for a larger budget, we expect smaller communication complexity. 

\subsection{Theoretical Assumptions}In this paper, we always assume the following four assumptions. Assumptions \ref{assump: local_loss_smoothness}, \ref{assump: optimal_sol} and \ref{assump: bounded_loss_gradient_variance} are fairly standard in first-order nonconvex optimization. 

\begin{assumption}[Heterogeneity]
\label{assump: heterogeneous}
$\{f_p\}_{p=1}^P$ is second-order  $\zeta$-heterogeneous, i.e., for any $p, p' \in [P]$, 
\begin{align*}
    \left\|\nabla^2 f_p(x) - \nabla^2 f_{p'}(x)\right\| \leq \zeta, \forall x \in \mathbb{R}^d.
\end{align*}
\end{assumption}
Assumption \ref{assump: heterogeneous} characterizes the heterogeneity of local objectives $\{f_p\}_{p=1}^P$ and has a critical role in our analysis. We expect that relatively small heterogeneity to the smoothness reduces the necessary number of communication to optimize global objective $f=(1/P)\sum_{p=1}^P f_p$. If the local objectives are identical, i.e., $D_p = D_{p'}$ for every $p, p' \in [P]$, $\zeta$ becomes zero. 
When each $D_p$ is the empirical distribution of $n/P$ IID samples from common data distribution $D$, we have $\|\nabla^2 f_p(x) - \nabla^2 f_{p'}(x)\| \leq \widetilde \Theta(\sqrt{P/n}L)$ with high probability by matrix Hoeffding's inequality under Assumption \ref{assump: local_loss_smoothness} for fixed $x$
\footnotemark.
\footnotetext{Although to show the high probability bound for every $x \in \mathbb{R}^d$ is generally difficult, we can use the high probability bounds on the discrete optimization path rather than the entire space $\mathbb{R}^d$ and then the same bound still holds. For only simplicity, we assume the heterogeneity condition on entire space $\mathbb{R}^d$ in this paper. }
Hence, in classical distributed learning regimes, Assumption \ref{assump: heterogeneous} naturally holds. An important remark is that {\bf{\emph{Assumption \ref{assump: local_loss_smoothness} implies $\zeta \leq 2 L$, i.e., the heterogeneity is bounded by the smoothness}}}. This means that \emph{Assumption \ref{assump: heterogeneous} gives an interpolation between the identical data setting $\zeta = 0$ and the extremely non-IID setting $\zeta = 2L$}. Even in federated learning regimes $\zeta \gg \sqrt{P/n}L$, 
we can expect $\zeta \ll 2L$ for some problems. 

\begin{assumption}[Smoothness]
\label{assump: local_loss_smoothness}
For any $p \in [P]$ and $z \in \mathrm{supp}(D_p)$, $\ell(\cdot, z)$ is $L$-smooth, i.e.,
$$\|\nabla \ell(x, z) - \nabla \ell(y, z)\| \leq L\|x - y\|, \forall x, y \in \mathbb{R}^d.$$
\end{assumption}
We assume $L$-smoothness of loss $\ell$ rather than risk $f$. This assumption is a bit strong, but is typically necessary in the analysis of variance reduced gradient estimators.  
\begin{assumption}[Existence of global optimum]
\label{assump: optimal_sol}
$f$ has a global minimizer $x_* \in \mathbb{R}^d$.
\end{assumption}

\begin{assumption}[Bounded gradient variance]
\label{assump: bounded_loss_gradient_variance}
For every $p \in [P]$, 
$$\mathbb{E}_{z\sim D_p}\|\nabla \ell(x, z) - \nabla f_p(x)\|^2 \leq \sigma^2.$$
\end{assumption}
Assumption \ref{assump: bounded_loss_gradient_variance} says that the variance of stochastic gradient is bounded for every local objective.

\section{Approach and Proposed Algorithms}
In this section, we introduce our approach and provide details of the proposed algorithms. \par
\begin{algorithm}[t]
\caption{Local GD($\widetilde x_0$, $\eta$, $B$, $b$, $K$, $T$)}
\label{alg: l_gd}                                                    
\begin{algorithmic}[1]
\FOR {$t=1$ to $T$}
\FOR {$p=1$ to $P$ \it{in parallel}}
\STATE Set $x_0^{(p)} = \widetilde x_{t-1}$.
\FOR {$k=1$ to $K$}
\STATE Update $x_k^{(p)} = x_{k-1}^{(p)} - \eta \nabla f_p(x_{k-1}^{(p)})$
\ENDFOR
\ENDFOR
\STATE Communicate $\{x_t^{(p)}\}_{p=1}^P$.
\STATE $\widetilde x_t = \frac{1}{P}\sum_{p=1}^P x_{\hat k}^{(p)}$ ($\hat k \sim \mathrm{Unif}[K]$).
\ENDFOR
\STATE {\bf{Return:}} $\widetilde x_{\hat t}$ ($\hat t \sim \mathrm{Unif}[T]$).
\end{algorithmic}
\end{algorithm}
\subsection{{Core Concepts and Approach}}
Here, we describe four main building blocks of our algorithm, that are \emph{localization}, \emph{bias reduction}, \emph{stochastization} and \emph{variance reduction}. Although our algorithm relies on SARAH like variance reduction technique, in this subsection we will describe our approach using SVRG like variance reduction rather than SARAH like one to simply convey the core ideas. \par
{\bf{Localization. }}One of the promising methods for reducing communication cost is local methods. In local methods, each worker independently optimizes the local objective and periodically communicate the current solution. For example, the algorithm of local GD, which is a deterministic variant of local SGD, is given in Algorithm \ref{alg: l_gd}. In some sense, \emph{the local gradient $\nabla f_p(x)$ can be regard as a biased estimator of the global gradient $\nabla f(x)$}. One of the limitations of local GD is the existence of the potential bias of the local gradient $\nabla f_p(x)$ to approximate the global one $\nabla f(x)$ for heterogeneous local objectives $\{f_p\}_{p=1}^P$. The bias $\|\nabla f_p(x) - \nabla f(x)\|$ critically affect the convergence speed and can be bounded as  $\|\nabla f_p(x) - \nabla f(x)\| \leq \zeta_1$ under the first order $\zeta_1$-heterogeneity condition. This implies that the bias heavily depends on the heterogeneity parameter $\zeta_1$ and does not converge to zero as $x \to x_*$. Hence, the existing analysis of local methods requires extremely small $\zeta_1$ that typically depends on the optimization accuracy $\varepsilon$ to surpass non-local methods including GD and minibatch SGD in terms of communication complexity, which is quite limited in many situations. \par
{\bf{Bias Reduction. }}To reduce the bias of local gradient, we introduce {\it{bias reduction}} technique. Concretely, we construct the local estimator $\nabla f_p(x) - \nabla f_p(x_0) + \nabla f(x_0)$ to approximate $\nabla f(x)$. Here, $x_0$ is the previously communicated solution. This construction evokes the famous variance reduction technique. Analogically to the analysis of variance reduced gradient estimators, under the second order $\zeta$-heterogeneity, the bias can be bounded as 
$\|\nabla f_p(x) - \nabla f_p(x_0) + \nabla f(x_0) - \nabla f(x)\| \leq \zeta\|x - x_0\|$. This means that the bias converges to zero as $x$ and $x_0 \to x_*$. Hence, \emph{the bias of the introduced estimator is reduced by utilizing the periodically computed global gradient $\nabla f(x)$}. This enables us to show faster convergence than vanilla non-local and local GD even for not too small $\zeta$. \par
{\bf{Stochastization. }}Generally, \emph{deterministic methods require huge computational cost for single update} in large scale optimization. The classical idea to handle this problem is
stochastization. For example, non-distributed SGD naively uses $\nabla \ell(x, z)$ with single sample $z \sim D_p$ to approximate $\nabla f(x) = \mathbb{E}_{z\sim D}[\nabla \ell(x, z)]$. \emph{Although stochastization reduces the computational cost per update, the variance due to it generally slows down the convergence speed}. Similar to standard SGD, we can naively stochastize our bias reduced estimator as $\ell(x, z) - \ell(x_0, z) + (1/P)\sum_{p'=1}^P\ell(x_0, z_{p'})$, where $z \sim D_p$ and $z_{p'} \sim D_{p'}$ for $p' \sim [P]$. Here, $\{z_{p'}\}_{p'=1}^P$ is sampled only at communication time. As pointed out before, the variance $\mathbb{E}_{z, z'\sim D_p}\|\ell(x, z) - \ell(x_0, z) + (1/P)\sum_{p'=1}^P\ell(x_0, z_{p'}) - (\nabla f_p(x) - \nabla f_p(x_0) + \nabla f(x_0))\|^2$ caused by stochastization may leads to slow convergence.
\par
{\bf{Variance Reduction. }}To reduce the variance of the gradient estimator due to stochastization, we introduce variance reduction technique. Variance reduction is also classical technique and has been extensively analysed both in convex and nonconvex optimization. The essence of variance reduction is the utilization of periodically computed full gradient $\nabla f(x)$. In non-distributed cases, a variance reduced estimator is defined as $\nabla \ell(x, z) - \nabla \ell(x_0, z) + \nabla f(x_0)$ with $z \sim D$. This estimator is unbiased and the variance $\mathbb{E}_{z\sim D}\|\nabla \ell(x, z) - \nabla \ell(x_0, z) + \nabla f(x_0) - \nabla f(x)\|^2$ can be bounded by $L^2\|x - x_0\|^2$, where $L$ is the smoothness parameter of $\ell$. If $x$ and $x_0 \to x_*$, the variance converges to zero. In this mean, \emph{the estimator reduces the variance caused by stochastization and also maintains computational efficiency by using periodically computed global full gradients}. Analogous to this formulation, each worker $p$ computes a variance reduced local gradient estimator $\nabla \ell(x, z) - \nabla \ell(x_0, z) + \nabla f(x_0)$ with $z \sim D_p$. 
\begin{algorithm}[t]
\caption{BVR-L-SGD($\widetilde x_0$, $\eta$, $b$, $\widetilde b$, $K$, $T$, $S$)}
\label{alg: bvr_l_sgd}
\begin{algorithmic}[1]
\FOR {$s=1$ to $S$}
    \FOR {$p=1$ to $P$ \it{in parallel}}
        \IF {$\widetilde b \geq                                 \frac{1}{P}\sum_{p=1}^P\#\mathrm{supp}(D_p)$} \label{line: local_full_grad}
        \STATE $\widetilde \nabla^{(p)} = \nabla f_p(\widetilde x_{s-1})$.
        \ELSE
        \STATE $\widetilde \nabla^{(p)} = \frac{1}{\widetilde b}\sum_{l=1}^{\widetilde b} \nabla \ell(\widetilde x_{s-1}, z_l)$ ($z_l \overset{i.i.d.}{\sim} D_p$).
        \ENDIF
    \ENDFOR
    \STATE Communicate $\{\widetilde \nabla^{(p)}\}_{p=1}^P$, set $\widetilde v_0 = \frac{1}{P}\sum_{p=1}^P \widetilde \nabla^{(p)}$.
    \STATE Set $x_0 = x_{-1} = \widetilde x_{s-1}$.
    \FOR {$t=1$ to $T$}
        \FOR {$p=1$ to $P$ in parallel}
        \STATE $g_t^{(p)}(x_{t-1}) = \frac{1}{Kb}\sum_{l=1}^{Kb} \nabla \ell(x_{t-1}, z_{l})$, 
        \STATE $g_t^{(p)}(x_{t-2}) =  \frac{1}{Kb}\sum_{l=1}^{Kb} \nabla \ell(x_{t-2}, z_{l})$ 
        \STATE for $z_l \overset{i.i.d.}{\sim} D_p$.
        \STATE $\widetilde v_t^{(p)} = g_t^{(p)}(x_{t-1}) - g_{t}^{(p)}(x_{t-2}) + \widetilde v_{t-1}^{(p)}$.
        \ENDFOR
        \STATE Communicate $\{\widetilde v_t^{(p)}\}_{p=1}^P$, set $\widetilde v_t = \frac{1}{P}\sum_{p=1}^P \widetilde v_t^{(p)}$.
        \FOR {$p=1$ to $P$ in parallel}
            \STATE $x_t^{(p)}$, $x_t^{(p), \mathrm{out}} =$ 
            \STATE \ \ \ \ \ \ \ \ \ Local-Routine$(p, x_{t-1}, \eta, \widetilde v_t, b, K)$. 
        \ENDFOR
        \STATE Communicate $\{x_t^{(p)}\}_{p=1}^P$ and $\{x_t^{(p), \mathrm{out}}\}_{p=1}^P$.
        \STATE Set $x_t = x_t^{(\hat p)}$ and  $x_t^{\mathrm{out}} = x_t^{(\hat p), \mathrm{out}}$ ($\hat p \sim \mathrm{Unif}[P]$).
    \ENDFOR
    \STATE Set $\widetilde x_s = x_T$ and $\widetilde x_s^{\mathrm{out}} = x_{\hat t}^{\mathrm{out}}$ ($\hat t \sim \mathrm{Unif}[T]$).
\ENDFOR
\STATE {\bf{Return:}} $\widetilde x^{\mathrm{out}} = \widetilde x_{\hat s}^{\mathrm{out}}$ ($\hat s \sim \mathrm{Unif}[S]$).
\end{algorithmic}
\end{algorithm}

\subsection*{{Concrete Algorithm}}
In this paragraph, we illustrate the concrete procedure of our proposed algorithm based on the concepts described in the previous paragraph. \par
The proposed algorithm for nonconvex objectives is provided in Algorithm \ref{alg: bvr_l_sgd}. In line 2-9, worker $p$ computes the full gradient of local objective $f_p$ (or a large batch stochastic gradient of $f_p$ if the learning problem is on-line, which means that $\# \mathrm{supp}(D_p) = \infty$ for some $p$). Then, each worker broadcasts it and the gradient of global objective $f$ is executed by averaging the communicated local gradients (line 9). Then, for each iteration $t$, \emph{each worker computes variance reduced local gradient $\widetilde v_t^{(p)}$ that approximates the full local gradient using $Kb$ IID samples} (line 13-16), that is an important process for computational efficiency. Then, $\{\widetilde v_t^{(p)}\}_{p=1}^P$ is communicated and $\widetilde v_t$ is obtained by averaging them. Using previous solution $x_{t-1}$ and $\widetilde v_t$ as inputs, each worker runs Local-Routine (Algorithm \ref{alg: local_update}) (line 21). The next solution $x_t$ at iteration $t$ is set to the randomly chosen solutions from Local-Routine's outputs $\{x_t\}_{p=1}^P$ rather than averaging them. When we terminate the for loop from line 11 to 22, the next solution at stage $s$ is set to the randomly chosen solutions from $\{x_t\}_{p=1, t=1}^{P, T}$ (line 24) rather than averaging them again. Although the model averaging process has a critical role in all the previous local algorithms, \emph{the random picking process is essential in our analysis to fully utilize the second-order heterogeneity and is one of the algorithmic novelties of our method}.
\par

\begin{algorithm}[t]
\caption{Local-Routine($p$, $x_0$, $\eta$, $v_0$, $b$, $K$)}
\label{alg: local_update}
\begin{algorithmic}[1]
\STATE Set $x_0^{(p)} = x_{-1}^{(p)} = x_0$.
\FOR {$k=1$ to $K$}
    \STATE $g_k^{(p)}(x_{k-1}^{(p)}) = \frac{1}{b}\sum_{l=1}^b \ell(x_{k-1}^{(p)}, z_{l})$, $g_{k}^{(p)}(x_{k-2}^{(p)}) =  \frac{1}{b}\sum_{l=1}^b \ell(x_{k-2}^{(p)}, z_{l})$ ($z_l \overset{i.i.d.}{\sim} D_p$).
    \STATE $v_k^{(p)} = g_k^{(p)}(x_{k-1}^{(p)}) - g_{k}^{(p)}(x_{k-2}^{(p)}) + v_{k-1}^{(p)}$.
    \STATE Update $x_k^{(p)} = x_{k-1}^{(p)} - \eta v_k^{(p)}$
\ENDFOR
\STATE {\bf{Return:}} $x_K^{(p)}$, $x_{\hat k}^{(p)}$ ($\hat k \sim \mathrm{Unif}[K]$).
\end{algorithmic}
\end{algorithm}

The local computation algorithm Local-Routine is illustrated in Algorithm \ref{alg: local_update}. In lines 3-5, we again use variance reduction with $v_0$ as a snapshot gradient. Here, \emph{we adopt the SARAH like variance reduction rather than the SVRG like one} because \emph{SARAH achieves the optimal computational complexity} for non-distributed nonconvex optimization.

\begin{remark}[Communication and computational complexity]
The communication complexity is $\Theta(ST)$ and the averaged number of single gradient computations per communication round for each worker is $\Theta((Kb + \widetilde b/T))$.
\end{remark}
\begin{remark}[Generalization of SARAH]
When $K=1$, BVR-L-SGD exactly matches to minibatch SARAH. In this sense, BVR-L-SGD is a generalization of minibatch SARAH.
\end{remark}
\begin{remark}[Practical Implementation]
Practically, in line 19-24 of Algorithm \ref{alg: bvr_l_sgd}, we randomly choose worker $\hat p$ \emph{at first} and execute Local-Routine only for worker $\hat p$. Note that this procedure gives an equivalent algorithm to the original one but reduces the computational and communication cost. More specific procedures of the practical implementation are found in the supplementary material (Section \ref{app_sec: implementation}).   
\end{remark}

\section{Convergence Analysis}
In this section, we provide theoretical convergence analysis of our proposed algorithm. For the proofs, see the supplementary material (Section \ref{app_sec: local_routine} and  \ref{app_sec: bvr_l_sgd}). \par
\subsection{{Analysis of Local-Routine}}
Here, we analyse Local-Routine (Algorithm \ref{alg: local_update}). 
\begin{lemma}[Descent Lemma]
\label{lem: descent}
Suppose that Assumption \ref{assump: local_loss_smoothness} holds. There exists $\eta_1 = \Theta(1/L)$ such that for any $\eta \leq \eta_1$, Local-Routine($p$, $x_0$, $\eta$, $v_0$, $b$, $K$) satisfies for $k \in [K]$,
\begin{align*}
    \mathbb{E}\|\nabla f(x_{k-1}^{(p)})\|^2 \leq&\  \Theta\left(\frac{1}{\eta}\right)(\mathbb{E}f(x_{k-1}^{(p)}) - \mathbb{E}f(x_k^{(p)})) \\
    &+ \frac{5}{2}\mathbb{E}\|v_k^{(p)} - \nabla f(x_{k-1}^{(p)})\|^2.
\end{align*}
\end{lemma}

The deviation of $v_k^{(p)}$ from $\nabla f(x_{k-1}^{(p)})$ can be bounded by the following lemma.   
\begin{lemma}\label{lem: local_grad_var_nonconvex}
Suppose that Assumptions \ref{assump: heterogeneous} and  \ref{assump: local_loss_smoothness} hold. Then, 
there exists $\eta_2 = \Theta(1/(K\zeta) \wedge \sqrt{b}/(\sqrt{K}L))$ such that for $\eta \leq \eta_2$, Local-Routine($p$, $x_0$, $\eta$, $v_0$, $b$, $K$)  satisfies 
\begin{align*}
    &\frac{1}{K}\sum_{k=1}^K\mathbb{E}\|v_k^{(p)} - \nabla f(x_{k-1}^{(p)})\|^2 \\
    \leq&\ \Theta\left(C_\eta\right)\sum_{k=1}^K \mathbb{E}\left\|\nabla f(x_{k-1}^{(p)})\right\|^2 + \Theta(1)\|v_0 - \nabla f(x_0)\|^2, 
\end{align*}
where $C_\eta := \eta^2L^2/b + \eta^2\zeta^2K$.
\end{lemma}

Combining Lemma \ref{lem: descent} and \ref{lem: local_grad_var_nonconvex} results in the following proposition. 
\begin{proposition}
\label{prop: local_grad_bound_nonconvex}
Suppose that Assumptions \ref{assump: heterogeneous} and  \ref{assump: local_loss_smoothness} hold. There exists $\eta_3 = \Theta(1/L \wedge 1/(K\zeta) \wedge \sqrt{b}/(\sqrt{K}L))$ such that for $\eta \leq \eta_3$, Local-Routine($p$, $x_0$, $\eta$, $v_0$, $b$, $K$)  satisfies 
\begin{align*}
    \mathbb{E}\|\nabla f(x_{\hat k}^{(p)})\|^2 \leq&\ \Theta\left(\frac{1}{\eta K}\right)(\mathbb{E}f(x_0) - \mathbb{E}f(x_K^{(p)}))\\
    &+ \Theta(1)\|v_0 - \nabla f(x_0)\|^2. 
\end{align*}
\end{proposition}

\subsection{{Analysis of BVR-L-SGD}}
Here, we analyse BVR-L-SGD (Algorithm \ref{alg: bvr_l_sgd}). The following lemma bounds the variance of $\widetilde v_t$, which arises in Proposition \ref{prop: local_grad_bound_nonconvex}.
\begin{lemma}
\label{lem: global_grad_var_bound_nonconvex}
Suppose that Assumptions \ref{assump: heterogeneous},  \ref{assump: local_loss_smoothness} and \ref{assump: bounded_loss_gradient_variance} hold. Then, 
there exists $\eta_4 = \Theta(1/(K\zeta) \wedge \sqrt{b}/(\sqrt{K}L) \wedge \sqrt{Pb}/(\sqrt{KT}L))$ such that for $\eta \leq \eta_4$, BVR-L-SGD($\widetilde x_0$, $\eta$, $b$, $\widetilde b$, $K$, $T$, $S$) satisfies 
\begin{align*}
    & \frac{1}{T}\sum_{t=1}^T\mathbb{E}\|\widetilde v_t - \nabla f(x_{t-1})\|^2 \\
    \leq&\ \Theta\left(C_\eta'\right)\frac{1}{T}\sum_{t'=1}^T\frac{1}{P}\sum_{p=1}^P\mathbb{E}[G_{t'}^{(p)}] + \mathds{1}_{\widetilde b < \frac{n}{P}}\frac{\sigma^2}{P\widetilde b}.
\end{align*}
where $C_\eta' := \eta^2L^2KT/(Pb) + \eta^2L^2K/b + \eta^2\zeta^2K^2$, $n := \sum_{p=1}^P\#\mathrm{supp}(D_p)$ and $G_{t'}^{(p)} := \|\nabla f(x_{t'-1}^{(p), \mathrm{out}})\|^2$.
\end{lemma}

Using the result of Local-Routine (Proposition \ref{prop: local_grad_bound_nonconvex}) and Lemma \ref{lem: global_grad_var_bound_nonconvex}, we obtain the following theorem.  
\begin{theorem}
\label{thm: nonconvex}
Suppose that Assumptions \ref{assump: heterogeneous},  \ref{assump: local_loss_smoothness}, \ref{assump: optimal_sol} and \ref{assump: bounded_loss_gradient_variance} hold. Then, 
there exists $\eta = \Theta(1/L \wedge 1/(K\zeta) \wedge \sqrt{b}/(\sqrt{K}L) \wedge \sqrt{Pb}/(\sqrt{KT}L))$ such that BVR-L-SGD($\widetilde x_0$, $\eta$, $b$, $\widetilde b$, $K$, $T$, $S$) satisfies 
\begin{align*}
    \mathbb{E}\|\nabla f(\widetilde x^{\mathrm{out}})\|^2 \leq&\  \Theta\left(\frac{1}{\eta KTS}\right)(\mathbb{E}f(\widetilde x_{0}) - f(x_*)) \\
    &+ \Theta(1)\mathds{1}_{\widetilde b < \frac{n}{P}}\frac{\sigma^2}{P\widetilde b},
\end{align*}
where $n := \sum_{p=1}^P\#\mathrm{supp}(D_p)$.
\end{theorem}

Theorem \ref{thm: nonconvex} immediately implies the following corollary which characterises the communication complexity of BVR-L-SGD. 
\begin{corollary}
\label{cor: complexity_nonconvex}
Suppose that Assumptions \ref{assump: heterogeneous},  \ref{assump: local_loss_smoothness}, \ref{assump: optimal_sol} and \ref{assump: bounded_loss_gradient_variance} hold. We denote $n = \sum_{p=1}^P\# \mathrm{supp}(D_p)$. Let $\widetilde b = \Theta((n/P) \wedge (\sigma^2/(P\varepsilon)))$. Then, 
there exists $\eta = \Theta(1/L \wedge 1/(K\zeta) \wedge \sqrt{b}/(\sqrt{K}L) \wedge \sqrt{Pb}/(\sqrt{KT}L)))$ such that BVR-L-SGD($\widetilde x_0$, $\eta$, $b$, $\widetilde b$, $K$, $T$, $S$) with $S = \Theta(1 + 1/(\eta KT\varepsilon))$ satisfies 
$\mathbb{E}\|\nabla f(\widetilde x^{\mathrm{out}})\|^2 \leq \Theta(\varepsilon)$ with communication complexity  
\begin{align*}
    ST = \Theta\left(\frac{L}{K\varepsilon} + \frac{\zeta}{\varepsilon} + \frac{L}{\sqrt{Kb}\varepsilon} + \sqrt{\frac{T}{KbP}}\frac{L}{\varepsilon} + T\right).
\end{align*}

\end{corollary}

\begin{remark}[Communication efficiency]
Given local computation budget $\mathcal B$, we set $T = \Theta(1 + \widetilde b/\mathcal B)$ and $Kb = \Theta(\mathcal B)$ with $b \leq \Theta(\sqrt{\mathcal B})$, where $\widetilde b$ was defined in Corollary \ref{cor: complexity_nonconvex}. Then, we have the averaged number of local computations per communication round $Kb + \widetilde b / T = \Theta(\mathcal B)$ and the total communication complexity with budget $\mathcal B$ becomes $\Theta((L/(\sqrt{\mathcal B}\varepsilon) + \sqrt{n\wedge (\sigma^2/\varepsilon)}L/(\mathcal BP\varepsilon) + (n\wedge (\sigma^2/\varepsilon))/(\mathcal B P) +  \zeta/\varepsilon)$. 
\end{remark}

\section{Numerical Resutls}
In this section, we provide several experimental results to verify our theoretical findings. \par
We conducted a ten-class classification on CIFAR10\footnotemark\footnotetext{\url{https://www.cs.toronto.edu/~kriz/cifar.html}.} dataset. Several heterogeneity patterns of local datasets were artificially created. For each heterogeneity, we compared the empirical performances of our method and several existing methods.\par
{\bf{Data Preparation. }}We first equalized the number of data per class by randomly removing the excess data for both the train and test datasets for only simplicity. Then, for fixed $q \in \{0.1, 0.35, 0.6, 0.85\}$, $q \times 100$ \% of the data of class $p$ was assigned to worker $p$ for $p \in [P]$. Here, we set the number of workers to the number of classes. Then, for each class $p$, we equally divided the remained $(1 - q) \times 100$ \% data of class $p$ into $P-1$ sets and distributed them to correspondence worker $p'\neq p$. As a result, we obtained several patterns of class imbalanced local datasets with various heterogeneity (we expect smaller heterogeneity for smaller $q$ and particularly $\zeta \approx 0$ when $q=0.1$ since $P=10$). An illustration of this process for $P=3$ is given in Figure \ref{fig: data_preparation}. From this process, we fixed the number of workers $P$ to ten. Finally, we normalized each channel of the data to be mean and standard deviation $0.5$. 

\begin{figure}[t]
  \includegraphics[width=8.0cm]{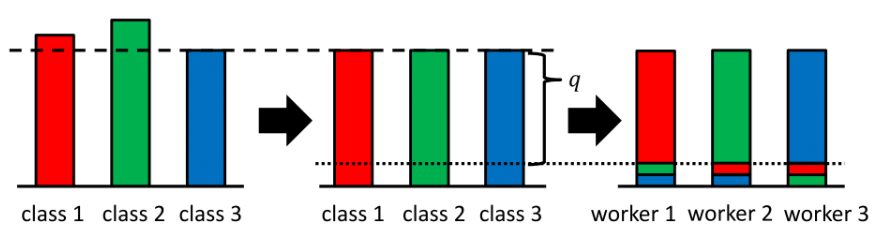}
  \caption{An illustraion of our local datasets generation process. }
  \label{fig: data_preparation}
\end{figure}

\par
{\bf{Models. }}We conducted our experiments using an one-hidden layer fully connected neural network with $100$ hidden units and softplus activation. For loss function, we used the standard cross-entropy loss. We initialized parameters by uniformly sampling the parameters from $[-\sqrt{6/(n_{\mathrm{in}} + n_{\mathrm{out}})}, \sqrt{6/(n_{\mathrm{in}} + n_{\mathrm{out}})}]$ \cite{glorot2010understanding}, where $n_{\mathrm{in}}$ and $n_{\mathrm{out}}$ are the number of units in the input and output layers respectively. Furthermore, we add $L_2$-regularizer to the empirical risk with fixed regularization parameter $5\times10^{-3}$. 
\par
{\bf{Implemented Algorithms. }} We implemented minibatch SGD, Local SGD, SARAH, SCAFFOLD and our BVR-L-SGD. For each local computation budget $\mathcal B \in \{256, 512, 1024\}$, we set $K = \mathcal B/16$ and $b = 16$ for local methods (Local SGD, SCAFFOLD and BVR-L-SGD), and $b = \mathcal B$ for non-local ones (minibatch SGD and SARAH). Note that each algorithm requires the same order of stochastic gradient computations per communication. For each algorithm, we tuned learning rate $\eta$ from $\{0.005, 0.01, 0.05, 0.1, 0.5, 1.0\}$. The details of the tuning procedure are found in the supplementary material (Section \ref{app_sec: experiment}).
\par

{\bf{Evaluation. }}
We compared the implemented algorithms using four criteria of train loss;  train accuracy; test loss and test accuracy against (i) heterogeneity $q$; (ii) local computation budgets $\mathcal B$ and (iii) the number of communication rounds. The total number of communication rounds was fixed to $3,000$ for each algorithm. We independently repeated the experiments $5$ times and report the mean and standard deviation of the above criteria. Due to the space limitation, we will only report train loss and test accuracy in the main paper. The full results are found in the supplementary material (Section \ref{app_sec: experiment}). 
\par

{\bf{Results 1: Effect of the heterogeneity. }} Here, we investigate the effect of the heterogeneity on the convergence speed of the algorithms. To clarify the pure effect of the heterogeneity, we fixed the local computation budget $\mathcal B$ to $1,024$, which was the largest one in our experiments. Figure \ref{fig: by_heterogeneities} shows the comparison of the best-achieved train loss and test accuracy in $3,000$ communication rounds against heterogeneity parameter $q$. From this, we can see that the convergence speed of the local methods deteriorated as heterogeneity parameter $q$ increased. Particularly, the degree of the performance degradation of L-SGD and SCAFFOLD was serious. In contrast, this phenomenon was not observed for non-local methods, because the convergence rates of non-local methods do not depend on heterogeneity $\zeta$ as in Table \ref{tab: theoretical_comparison}.  Importantly, even for the largest $q$, BVR-L-SGD significantly outperformed the other methods.  \par
\begin{figure}[t]
\begin{subfigmatrix}{2}
\subfigure[Best Train Loss]{\includegraphics[width=4cm]{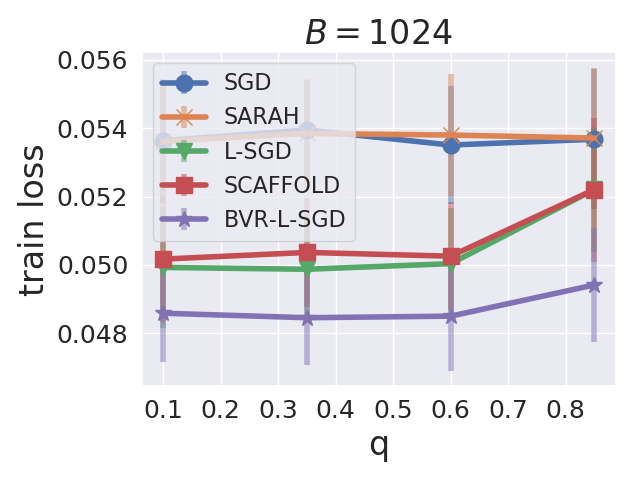}}
\subfigure[Best Test Accuracy]{\includegraphics[width=4cm]{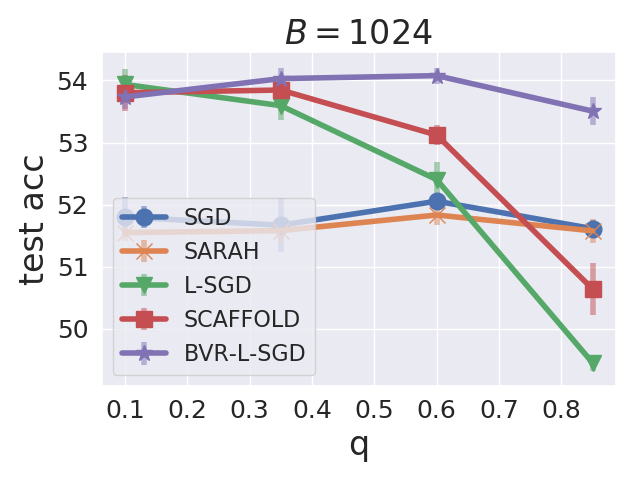}}
\end{subfigmatrix}
\caption{Comparison of the best train loss and test accuracy against heterogeneity parameter $q$. }
\label{fig: by_heterogeneities}
\end{figure}
{\bf{Results 2: Effect of the local computation budget size. }} Now, we examine the effect of the size of the local computation budget $\mathcal B$ to the convergence speed. For this purpose, we fixed heterogeneity parameter $q$ to the smallest one. Figure \ref{fig: by_budgets} shows the comparison of the best-achieved train loss and test accuracy against local computation budget $\mathcal B$. We can see that the local methods improved their performances as local computation budget $\mathcal B$ increased, but non-local methods did not. This is because local methods can potentially achieve a smaller communication complexity than $1/\varepsilon$ by increasing $\mathcal B$ for small $\zeta$, but non-local methods can not break the barrier of $1/\varepsilon$ for any $\mathcal B$ as in Table \ref{tab: theoretical_comparison}. \par
\begin{figure}[t]
\begin{subfigmatrix}{2}
\subfigure[Best Train Loss]{\includegraphics[width=4cm]{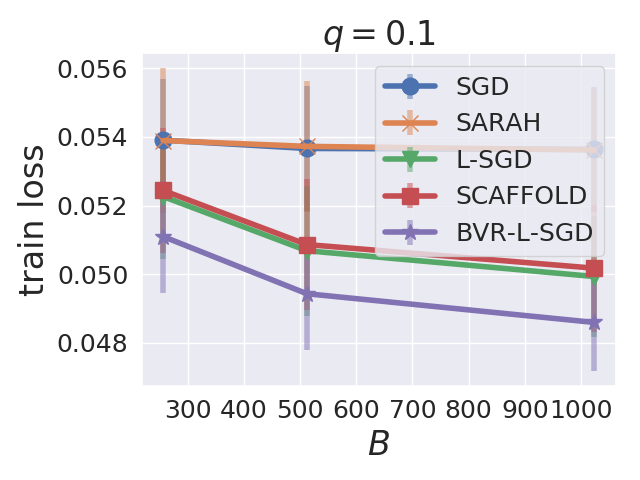}}
\subfigure[Best Test Accuracy]{\includegraphics[width=4cm]{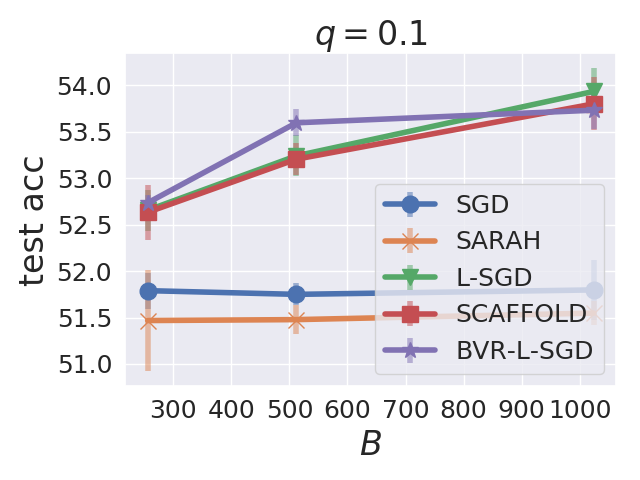}}
\end{subfigmatrix}
\caption{Comparison of the best train loss and test accuracy against local computation budget $\mathcal B$. }
\label{fig: by_budgets}
\end{figure}

{\bf{Results 3: Effect of the number of communication rounds.}} Finally, to see the trends of train loss and test accuracy during optimization, we give the comparison of the train loss and test accuracy against the number of communication rounds (Figure \ref{fig: by_rounds}). For the space limitation, we only report the case $\mathcal B = 1,024$. From these results, it can be seen that our proposed BVR-L-SGD consistently outperformed the other methods from beginning to end. \par

\begin{figure}[t]
\begin{subfigmatrix}{2}
\subfigure[Train Loss]{\includegraphics[width=4cm]{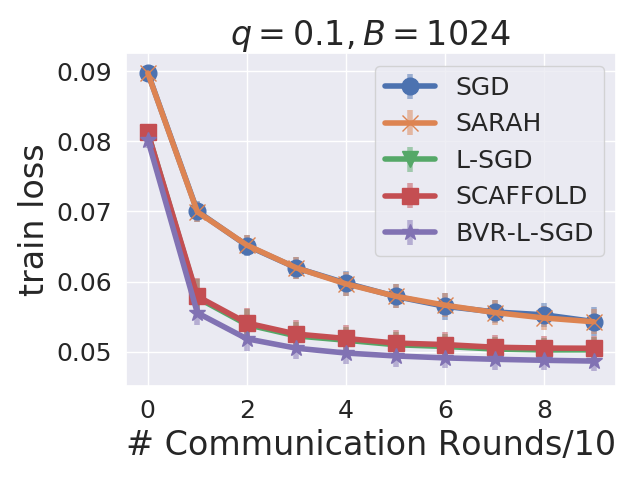}}
\subfigure[Test Accuracy]{\includegraphics[width=4cm]{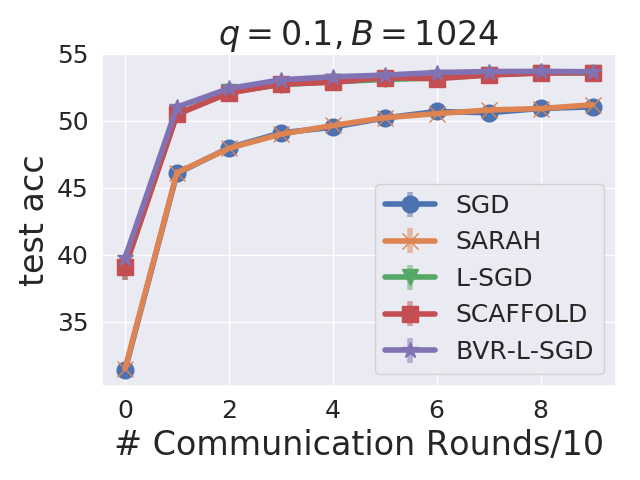}}
\end{subfigmatrix}
\begin{subfigmatrix}{2}
\subfigure[Train Loss]{\includegraphics[width=4cm]{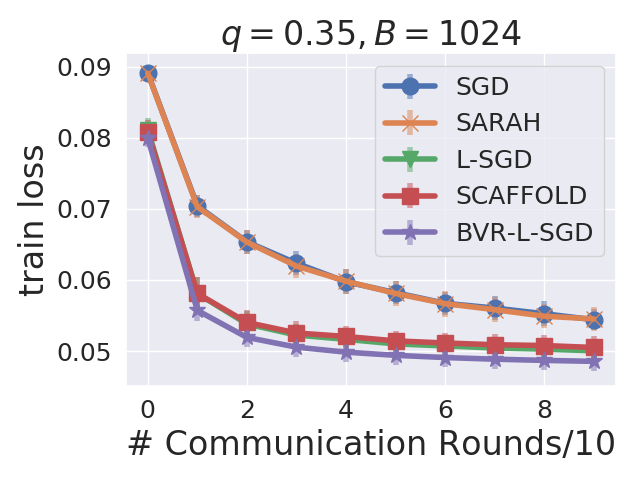}}
\subfigure[Test Accuracy]{\includegraphics[width=4cm]{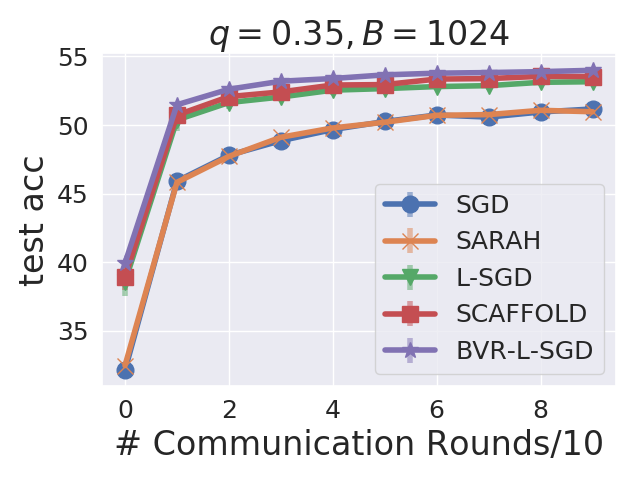}}
\end{subfigmatrix}
\begin{subfigmatrix}{2}
\subfigure[Train Loss]{\includegraphics[width=4cm]{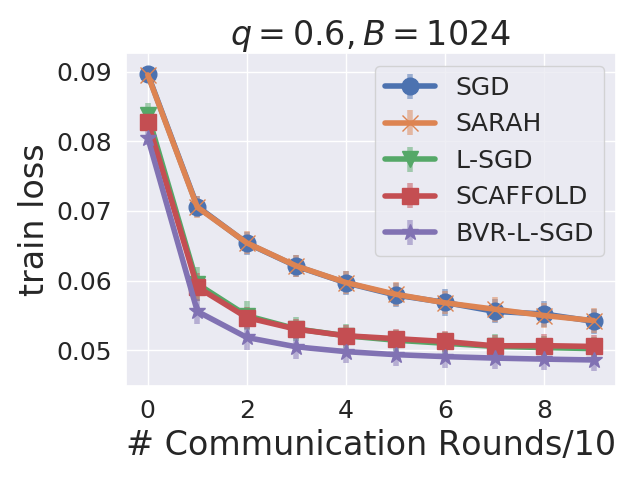}}
\subfigure[Test Accuracy]{\includegraphics[width=4cm]{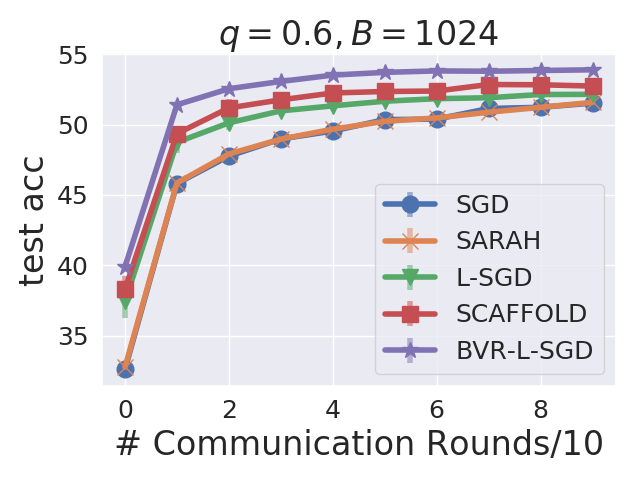}}
\end{subfigmatrix}
\begin{subfigmatrix}{2}
\subfigure[Train Loss]{\includegraphics[width=4cm]{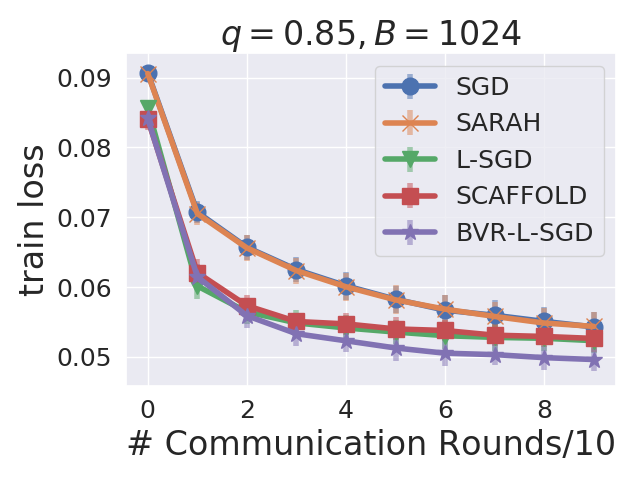}}
\subfigure[Test Accuracy]{\includegraphics[width=4cm]{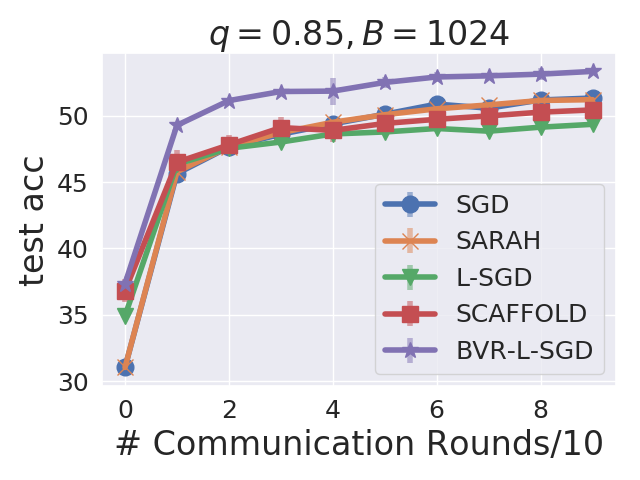}}
\end{subfigmatrix}
\caption{Comparison of the train loss and test accuracy against the number of communication rounds. }
\label{fig: by_rounds}
\end{figure}
In summary, for small heterogeneity, local methods significantly surpassed non-local methods. For relatively large heterogeneity, the performances of the existing local methods were seriously degraded. In contrast, the degree of deterioration of BVR-L-SGD was small and BVR-L-SGD consistently out-performed both the existing non-local and local methods. These observations strongly verify the theoretical findings and showed the empirical superiority of our method.

\section{Conclusion and Future Work}
In this paper, we studied our proposed BVR-L-SGD for nonconvex distributed learning, which is based on the bias-variance reduced gradient estimator to fully utilize the small second-order heterogeneity of local objectives and suggests randomly picking up one of the local models instead of taking the average of them when workers are synchronized. Our theory implies the superiority of BVR-L-SGD to previous non-local and local methods in terms of communication complexity.   The numerical results strongly encouraged our theoretical results and suggested the empirical superiority of the proposed method. \par
One promising and challenging future work is to extend our algorithm and theory to the problem of finding second-order stationary points. Although there are many papers for finding second-order stationary points for general nonconvex problems \cite{ge2015escaping, allen2017natasha, jin2017escape, li2019ssrgd}, it might be inherently difficult for local algorithms to efficiently find a local minima due to the nature of local updates. An open question is that: can we construct a local algorithm that guarantees to find second-order stationary points and is more communication efficient than non-local methods for local objectives with small heterogeneity?

\section*{Acknowledgement}
TS was partially supported by JSPS KAKENHI (18K19793, 18H03201, and 20H00576), Japan
DigitalDesign, and JST CREST.

\bibliography{main}
\bibliographystyle{icml2021}

\appendix

\onecolumn

\section{Analysis of Local-Routine (Algorithm \ref{alg: local_update})}
\label{app_sec: local_routine}
In this section, we give the analysis of Local-Routine. 
\subsection*{Proof of Lemma \ref{lem: descent}}

Let $\eta_1 = 1/(4L)$.
From $L$-smoothness of $f$, we have
\begin{align*}
    f(x_k^{(p)}) \leq&\ f(x_{k-1}^{(p)}) + \langle \nabla f(x_{k-1}^{(p)}), x_k^{(p)} - x_{k-1}^{(p)}\rangle + \frac{L}{2}\|x_k^{(p)} - x_{k-1}^{(p)}\|^2 \\
    =&\ f(x_{k-1}^{(p)}) - \eta \langle \nabla f(x_{k-1}^{(p)}), v_k^{(p)}\rangle + \frac{\eta^2 L}{2}\|v_k^{(p)}\|^2 \\
    \leq&\ f(x_{k-1}^{(p)}) - \eta\|\nabla f(x_{k-1}^{(p)})\|^2 - \eta \langle \nabla f(x_{k-1}^{(p)}), v_k^{(p)} - \nabla f(x_{k-1}^{(p)})\rangle + \eta^2L\|v_k^{(p)} - \nabla f(x_{k-1}^{(p)})\|^2 \\
    &+ \eta^2L\|\nabla f(x_{k-1}^{(p)})\|^2 \\
    \leq&\  f(x_{k-1}^{(p)}) - \eta\left(\frac{3}{4} - \frac\eta L\right)\|\nabla f(x_{k-1}^{(p)})\|^2 + \eta\left(1 + \eta L\right)\|v_k^{(p)} - \nabla f(x_{k-1}^{(p)})\|^2 \\
    \leq&\ f(x_{k-1}^{(p)}) - \frac{\eta}{2}\|\nabla f(x_{k-1}^{(p)})\|^2 + \eta\left(1 + \eta L\right)\|v_k^{(p)} - \nabla f(x_{k-1}^{(p)})\|^2.
\end{align*}
Here, in the second and third inequalities we used Cauchy Schwarz inequality and Arithmetic Mean-Geometric Mean inequality. The last inequality holds because $\eta \leq 1/(4L)$. Hence, we get
\begin{align*}
    \|\nabla f(x_{k-1}^{(p)})\|^2 \leq&\ \frac{2}{\eta}(f(x_{k-1}^{(p)}) - f(x_k^{(p)})) + 2(1 + \eta L)\|v_k^{(p)} - \nabla f(x_{k-1}^{(p)})\|^2 \\
    \leq&\ \frac{2}{\eta}(f(x_{k-1}^{(p)}) - f(x_k^{(p)})) + \frac{5}{2}\|v_k^{(p)} - \nabla f(x_{k-1}^{(p)})\|^2.
\end{align*}
Finally, taking expectation on both sides yields the desired result. 
\qed

\begin{lemma}\label{lem: sol_dev}
Local-Routine($p$, $x_0$, $\eta$, $v_0$, $b$, $K$) satisfies  for $k \in [K]$,
\begin{align*}
    \mathbb{E}\|x_k^{(p)} - x_0\|^2 \leq \Theta(\eta^2K^2)\frac{1}{K}\sum_{k'=1}^K \mathbb{E}\|v_{k'}^{(p)} - \nabla f(x_{k'-1}^{(p)})\|^2 + \Theta(\eta^2 K^2) \frac{1}{K} \sum_{k'=1}^K \mathbb{E}\|\nabla f(x_{k'-1}^{(p)})\|^2.
\end{align*}
for $p \in [P]$ and $k \in [K]$.
\end{lemma}
\begin{proof}
\begin{align*}
    &\|x_k^{(p)} - x_0\|^2 \\
    =&\ \|x_{k-1}^{(p)} - x_0 + \eta v_k^{(p)}\|^2 \\
    \leq&\ \left(1+\frac{1}{K}\right)\|x_{k-1}^{(p)} - x_0\|^2 + \eta^2 (1+K)\|v_k^{(p)}\|^2 \\
    \leq&\ \left(1+\frac{1}{K}\right)\|x_{k-1}^{(p)} - x_0\|^2 + 2\eta^2 (1+K)\|v_k^{(p)} - \nabla f(x_{k-1}^{(p)})\|^2 + 2\eta^2 K \|\nabla f(x_{k-1}^{(p)})\|^2. 
\end{align*}
Here, the inequality follows from Cauchy Schwarz inequality and Arithmetic Mean-Geometric Mean inequality.  
Recursively using this inequality, we obtain
\begin{align*}
    &\|x_k^{(p)} - x_0\|^2 \\
    \leq&\  2\eta^2(1+K) \sum_{k'=1}^k \left(1+\frac{1}{K}\right)^{k-k'}\|v_{k'}^{(p)} - \nabla f(x_{k'-1}^{(p)})\|^2 + 2\eta^2 (1+K) \sum_{k'=1}^k \left(1+\frac{1}{K}\right)^{k-k'}\|\nabla f(x_{k'-1}^{(p)})\|^2 \\
    \leq&\ 2e\eta^2(1+K) \sum_{k'=1}^K \|v_{k'}^{(p)} - \nabla f(x_{k'-1}^{(p)})\|^2 + 2e\eta^2 (1+K) \sum_{k'=1}^K \|\nabla f(x_{k'-1}^{(p)})\|^2.
\end{align*}
Here, we used the fact that $(1+1/K)^{k-k'} \leq (1+1/K)^K \leq e$ and the definition $x_0^{(p)} = x_0$.
\end{proof}

\begin{lemma}\label{lem: hetero_var}
Suppose that Assumptions and \ref{assump: heterogeneous} hold. Then, for any $x, y \in \mathbb{R}^d$, 
$$\|\nabla f_p(x) - \nabla f_p(y) + \nabla f(y) - \nabla f(x)\|^2 \leq \zeta^2\|x -y\|^2$$
for $p \in [P]$.
\end{lemma}
\begin{proof}
From the convexity of $\|\cdot\|^2$, we have
\begin{align*}
    &\|\nabla f_p(x) - \nabla f_p(y) + \nabla f(y) - \nabla f(x)\|^2\\
    \leq&\ \frac{1}{P}\sum_{p'\neq p} \|\nabla f_p(x) - \nabla f_p(y) + \nabla f_{p'}(y) - \nabla f_{p'}(x)\|^2.
\end{align*}

Since $f_p - f_{p'}$ is $C^2$- function, $\nabla f_p (x) - \nabla f_{p'}(x) - \nabla f_p(y) + \nabla f_{p'}(y) = \nabla (f_p - f_{p'}) (x) - \nabla (f_p - f_{p'}) (y) = (\nabla^2 (f_p - f_{p'})(\xi)) (x - y)$ for some $\xi \in \mathbb{R}^d$ by Mean value theorem. Hence, we have
\begin{align*}
    &\|\nabla f_p(x) - \nabla f_p(y) + \nabla f_{p'}(y) - \nabla f_{p'}(x)\|^2 \\
    \leq&\ \frac{1}{P}\sum_{p'\neq p}\|\nabla^2 f_p(\xi) - \nabla^2 f_{p'}(\xi)\|_2^2 \|x - y\|^2 \\
    \leq&\ \zeta^2\|x - y\|^2.
\end{align*}
Here the last inequality holds thanks to Assumption \ref{assump: heterogeneous}. 
\end{proof}

\subsection*{Proof of Lemma \ref{lem: local_grad_var_nonconvex}}
Observe that 
\begin{align*}
    &\mathbb{E}\|v_k^{(p)} - \nabla f(x_{k-1}^{(p)})\|^2 \\
    =&\ \mathbb{E}\|g_k^{(p)}(x_{k-1}^{(p)}) - g_{k}^{(p)}(x_{k-2}^{(p)})  + v_{k-1}^{(p)} - \nabla f(x_{k-1}^{(p)})\|^2 \\
    =&\ \mathbb{E}\|g_k^{(p)}(x_{k-1}^{(p)}) - g_{k}^{(p)}(x_{k-2}^{(p)}) - \nabla f_p(x_{k-1}^{(p)}) + \nabla f_p(x_{k-2}^{(p)})\|^2 \\
    &+ \mathbb{E}\|\nabla f_p(x_{k-1}^{(p)}) - \nabla f_p(x_{k-1}^{(p)}) + v_{k-1}^{(p)} - \nabla f(x_{k-1}^{(p)})\|^2 \\
    \leq&\ \mathbb{E}\|g_k^{(p)}(x_{k-1}^{(p)}) - g_{k}^{(p)}(x_{k-2}^{(p)})  - \nabla f_p(x_{k-1}) + \nabla f_p(x_{k-2}^{(p)})\|^2 \\
    &+ (1+K)\mathbb{E}\|\nabla f_p(x_{k-1}) - \nabla f_p(x_{k-2}^{(p)}) + \nabla f(x_{k-2}^{(p)}) - \nabla f(x_{k-1}^{(p)})\|^2 \\
    &+ \left(1+\frac{1}{K}\right)\mathbb{E}\|v_{k-1}^{(p)} - \nabla f(x_{k-2}^{(p)})\|^2 \\
    \leq&\ \frac{1}{b}\mathbb{E}\left[\mathbb{E}_{z\sim \mathcal{D}_p}\|\nabla \ell(x_{k-1}^{(p)}, z) - \nabla \ell(x_{k-2}^{(p)}, z)\|^2\right] \\
    &+ (1+K)\mathbb{E}\|\nabla f_p(x_{k-1}^{(p)}) - \nabla f_p(x_{k-2}^{(p)}) + \nabla f(x_{k-2}^{(p)}) - \nabla f(x_{k-1}^{(p)})\|^2 \\
    &+ \left(1+\frac{1}{K}\right)\mathbb{E}\|v_{k-1}^{(p)} - \nabla f(x_{k-2}^{(p)})\|^2.
\end{align*}
Here, the second equality holds because $\mathbb{E}[g_k^{(p)}(x_{k-1}^{(p)})] = \nabla f_p(x_{k-1}^{(p)})$ and $\mathbb{E}[g_k^{(p)}(x_{k-2}^{(p)})] = \nabla f_p(x_{k-2}^{(p)})$. The fist inequality is from Cauchy-Schwarz inequlality and Arithmetic Mean-Geometric Mean inequality. The last inequality follows from the fact that $g_k^{(p)}$ constituted by $b$ IID stochastic gradients. 
Recursively using this inequality, we have     
\begin{align*}
    &\mathbb{E}\|v_k^{(p)} - \nabla f(x_{k-1}^{(p)})\|^2 \\
    \leq&\ \frac{eK}{b}\frac{1}{K}\sum_{k'=1}^K\mathbb{E}\left[\mathbb{E}_{z\sim \mathcal{D}_p}\|\nabla \ell(x_{k'-1}^{(p)}, z) - \nabla \ell(x_{k'-2}^{(p)}, z)\|^2\right] \\ &+e(1+K)K \frac{1}{K}\sum_{k'=1}^K\mathbb{E}\|\nabla f_p(x_{k'-1}^{(p)}) - \nabla f_p(x_{k'-2}^{(p)}) + \nabla f(x_{k'-2}^{(p)}) - \nabla f(x_{k'-1}^{(p)})\|^2 \\
    &+ e\|v_0 - \nabla f(x_0)\|^2.
\end{align*}
Note that $x_0^{(p)} = x_0$ and $v_0^{(p)} = v_0$.
Then applying Lemma \ref{lem: hetero_var}, we get
\begin{align*}
    \mathbb{E}\|v_k^{(p)} - \nabla f(x_{k-1}^{(p)})\|^2 \leq&\ \Theta\left(\frac{K}{b}\right)\frac{1}{K}\sum_{k'=1}^K\mathbb{E}[\mathbb{E}_{z\sim \mathcal{D}_p}\|\nabla \ell(x_{k'-1}^{(p)}, z) - \nabla \ell(x_{k'-2}^{(p)}, z)\|^2] \\
    &+\Theta(\zeta^2K^2)\frac{1}{K}\sum_{k'=1}^K\mathbb{E}\|x_{k'-1}^{(p)} - x_{k'-2}^{(p)}\|^2 \\
    &+ \Theta(1)\|v_0 - \nabla f(x_0)\|^2 \\
    \leq&\  \Theta\left(\frac{L^2K}{b} + \zeta^2K^2\right)\frac{1}{K}\sum_{k'=1}^K\mathbb{E}\|x_{k'-1}^{(p)} - x_{k'-2}^{(p)}\|^2 \\
    &+ \Theta(1)\|v_0 - \nabla f(x_0)\|^2 \\
    \leq&\  \Theta\left(\frac{\eta^2L^2K}{b} + \eta^2\zeta^2K^2\right)\frac{1}{K}\sum_{k'=1}^K\mathbb{E}\|v_{k'-1}^{(p)} - \nabla f(x_{k'-1}^{(p)})\|^2 \\
    &+ \Theta\left(\frac{\eta^2L^2K}{b} + \eta^2\zeta^2K^2\right)\frac{1}{K}\sum_{k'=1}^K\mathbb{E}\|\nabla f(x_{k'-1}^{(p)})\|^2 \\
    &+ \Theta(1)\|v_0 - \nabla f(x_0)\|^2.
\end{align*}
Here, The second inequality holds by Assumption \ref{assump: local_loss_smoothness}. Averaging this inequality from $k=1$ to $K$ and choosing sufficiently small $\eta_2$ such that $\eta_2 = \Theta(1/(K\zeta)\wedge \sqrt{b}/(\sqrt{K}L))$, for any $\eta \leq \eta_2'$, the factor $\Theta(\eta^2L^2K/b + \eta^2\zeta^2K^2)$ becomes smaller than $1/2$. 
This gives the desired result. \qed

\subsection*{Proof of Proposition \ref{prop: local_grad_bound_nonconvex}}
From Lemma \ref{lem: descent}, we have
\begin{align*}
    \frac{1}{K}\sum_{k=1}^K \mathbb{E}\|\nabla f(x_{k-1}^{(p)})\|^2 \leq \Theta\left(\frac{1}{\eta K}\right)(\mathbb{E}f(x_0) - \mathbb{E}f(x_K^{(p)})) + \Theta(1)\frac{1}{K}\sum_{k=1}^K \mathbb{E}\|v_k^{(p)} - \nabla f(x_{k-1}^{(p)})\|^2. 
\end{align*}
Applying Lemma \ref{lem: local_grad_var_nonconvex} to this inequality, there exists $\eta_3  = \Theta((1/L) \wedge \eta_2)$, where $\eta_2$ is defined in Lemma \ref{lem: local_grad_var_nonconvex}, such that for every $\eta \leq \eta_3$, we get
\begin{align*}
    \frac{1}{K}\sum_{k=1}^K \mathbb{E}\|\nabla f(x_{k-1}^{(p)})\|^2 \leq \Theta\left(\frac{1}{\eta K}\right)(\mathbb{E}f(x_0) - \mathbb{E}f(x_K^{(p)})) + \Theta(1)\|v_0 - \nabla f(x_0)\|^2. 
\end{align*}
Finally, since $\hat k \sim \mathrm{Unif}[K]$, taking expectation with respect to $\hat k$ gives the desired result. \qed

\section{Analysis of BVR-L-SGD (Algorithm \ref{alg: bvr_l_sgd})}
\label{app_sec: bvr_l_sgd}
In this section, we provide the analysis of BVR-L-SGD.
\subsection*{Proof of Lemma \ref{lem: global_grad_var_bound_nonconvex}}
We define $V_t^{(p)}$ as $\frac{1}{K}\sum_{k=1}^K\mathbb{E}\|v_k^{(p)} - \nabla f(x_{k-1}^{(p)})\|^2$ in Local-Routine at iteration $t$. Then, we can rewrite the statement in Lemma \ref{lem: local_grad_var_nonconvex} as 
\begin{align*}
    V_t^{(p)} \leq \Theta\left(\frac{\eta^2L^2K}{b} + \eta^2\zeta^2K^2\right)\mathbb{E}\left\|\nabla f(x_t^{(p), \mathrm{out}})\right\|^2 + \Theta(1)\|\widetilde v_t - \nabla f(x_{t-1})\|^2. 
\end{align*}
Averaging this inequality from $p=1$ to $P$, we have
\begin{align}
    \frac{1}{P}\sum_{p=1}^PV_t^{(p)} \leq \Theta\left(\frac{\eta^2L^2K}{b} + \eta^2\zeta^2K^2\right)\frac{1}{P}\sum_{p=1}^P\mathbb{E}\left\|\nabla f(x_t^{(p), \mathrm{out}})\right\|^2 + \Theta(1)\|\widetilde v_t - \nabla f(x_{t-1})\|^2. 
    \label{ineq: global_var_bound}
\end{align}
Observe that 
\begin{align*}
    &\mathbb{E}\|\widetilde v_t - \nabla f(x_{t-1})\|^2 \\
    =&\ \mathbb{E}\left\|\frac{1}{P}\sum_{p=1}^P\left( g_t^{(p)}(x_{t-1}) - g_t^{(p)}(x_{t-2}) + \widetilde v_{t-1}^{(p)}\right) - \nabla f(x_{t-1})\right\|^2 \\
    =&\ \mathbb{E}\left\|\frac{1}{P}\sum_{p=1}^P\left( g_t^{(p)}(x_{t-1}) - g_t^{(p)}(x_{t-2})\right) - \nabla f(x_{t-1}) + \nabla f(x_{t-2})\right\|^2 + \mathbb{E}\|\widetilde v_{t-1} - \nabla f(x_{t-2})\|^2 \\
    =&\ \frac{1}{P^2}\sum_{p=1}^P\mathbb{E}\left\|g_t^{(p)}(x_{t-1}) - g_t^{(p)}(x_{t-2}) - \nabla f_p(x_{t-1}) + \nabla f_p(x_{t-2})\right\|^2 + \mathbb{E}\|\widetilde v_{t-1} - \nabla f(x_{t-2})\|^2 \\
    \leq&\ \frac{1}{PKb}\frac{1}{P}\sum_{p=1}^P\mathbb{E}\left[\mathbb{E}_{z\sim D_p}\left\|\nabla \ell(x_{t-1}, z) - \nabla \ell(x_{t-2}, z)\right\|^2\right] + \mathbb{E}\|\widetilde v_{t-1} - \nabla f(x_{t-2})\|^2.
\end{align*}
Here, the second inequality holds from $\mathbb{E}[g_t^{(p)}(x_{t-1})|t-1] = \nabla f_p(x_{t-1})$ and $\mathbb{E}[g_t^{(p)}(x_{t-2})|t-1] = \nabla f_p(x_{t-2})$. The last equality is from the independency of $g_t^{(p)}(x_{t-1}) - g_t^{(p)}(x_{t-2})$ given the history of the iterations $1, \ldots, t-1$. The last inequality holds because the $Kb$ samples used for $g_t^{(p)}(x_{t-1})$ and $g_t^{(p)}(x_{t-2})$ are IID.  
Recursively using this inequality, we have
\begin{align*}
    &\mathbb{E}\|\widetilde v_t - \nabla f(x_{t-1})\|^2 \\
    \leq&\ \frac{T}{PKb}\frac{1}{T}\sum_{t'=1}^T\frac{1}{P}\sum_{p=1}^P\mathbb{E}\left[\mathbb{E}_{z\sim D_p}\left\|\nabla \ell(x_{t'-1}, z) - \nabla \ell(x_{t'-2}, z)\right\|^2\right] + \mathbb{E}\|\widetilde v_{0} - \nabla f(x_0)\|^2 \\
    \leq&\ \frac{L^2T}{PKb}\frac{1}{T}\sum_{t'=1}^T\mathbb{E}\left\|x_{t'-1} - x_{t'-2}\right\|^2 + \mathds{1}_{\widetilde b < \frac{1}{P}\sum_{p=1}^P\#\mathrm{supp}( D_p)}\frac{\sigma^2}{P\widetilde b}.
\end{align*}
The last inequality follows from Assumptions \ref{assump: local_loss_smoothness} and \ref{assump: bounded_loss_gradient_variance} with the definition of $\widetilde v_0$.
From Lemma \ref{lem: sol_dev}, we have
\begin{align*}
    \mathbb{E}\|x_{t'} - x_{t'-1}\|^2 \leq \Theta(\eta^2K^2) \frac{1}{P}\sum_{p=1}^P(V_{t'}^{(p)} + \mathbb{E}\|\nabla f(x_{t'}^{(p), \mathrm{out}})\|^2).
\end{align*}
Hence, we get
\begin{align*}
    &\frac{1}{T}\sum_{t=1}^T\mathbb{E}\|\widetilde v_t - \nabla f(x_{t-1})\|^2 \\
    \leq&\ \frac{L^2T}{PKb}\frac{1}{T}\sum_{t'=1}^T\mathbb{E}\left\|x_{t'-1} - x_{t'-2}\right\|^2 + \mathds{1}_{\widetilde b < \frac{1}{P}\sum_{p=1}^P\#\mathrm{supp}( D_p)}\frac{\sigma^2}{P\widetilde b} \\
    \leq&\ \frac{\eta^2L^2KT}{Pb}\frac{1}{T}\sum_{t'=1}^T\frac{1}{P}\sum_{p=1}^PV_{t'}^{(p)} + \frac{\eta^2L^2KT}{Pb}\frac{1}{T}\sum_{t'=1}^T\frac{1}{P}\sum_{p=1}^P\mathbb{E}\|\nabla f(x_{t'-1}^{(p), \mathrm{out}})\|^2 + \mathds{1}_{\widetilde b < \frac{1}{P}\sum_{p=1}^P\#\mathrm{supp}( D_p)}\frac{\sigma^2}{P\widetilde b}.
\end{align*}
Choosing $\eta_4 \leq \eta_3$ such that $\Theta(\eta_4^2L^2 KT/Pb) < 1/2$, for every $\eta \leq \eta_4$, combining (\ref{ineq: global_var_bound}) yields 
\begin{align*}
    &\frac{1}{T}\sum_{t=1}^T\mathbb{E}\|\widetilde v_t - \nabla f(x_{t-1})\|^2 \\
    \leq&\  \Theta\left(\frac{\eta^2L^2K}{b} + \eta^2\zeta^2K^2 +  \frac{\eta^2L^2KT}{Pb}\right)\frac{1}{T}\sum_{t'=1}^T\frac{1}{P}\sum_{p=1}^P\mathbb{E}\|\nabla f(x_{t'-1}^{(p), \mathrm{out}})\|^2 + \mathds{1}_{\widetilde b < \frac{1}{P}\sum_{p=1}^P\#\mathrm{supp}( D_p)}\frac{\sigma^2}{P\widetilde b}.
\end{align*}
This is the desired result. \qed

\subsection*{Proof of Theorem \ref{thm: nonconvex}}
The statement of Proposition \ref{prop: local_grad_bound_nonconvex} at iteration $t$ implies
\begin{align*}
    \frac{1}{P}\sum_{p=1}^P\mathbb{E}\|\nabla f(x_t^{(p), \mathrm{out}})\|^2 \leq \Theta\left(\frac{1}{\eta K}\right)(\mathbb{E}f(x_{t-1}) - \mathbb{E}f(x_t)) + \Theta(1)\mathbb{E}\|\widetilde v_t - \nabla f(x_{t-1})\|^2.
\end{align*}
Averaging this inequality from $t=1$ to $T$ results in
\begin{align*}
    \frac{1}{T}\sum_{t=1}^T\frac{1}{P}\sum_{p=1}^P\mathbb{E}\|\nabla f(x_t^{(p), \mathrm{out}})\|^2 \leq \Theta\left(\frac{1}{\eta TK}\right)(\mathbb{E}f(x_{0}) - \mathbb{E}f(x_T)) + \Theta(1)\frac{1}{T}\sum_{t=1}^T\mathbb{E}\|\widetilde v_t - \nabla f(x_{t-1})\|^2.
\end{align*}
Then, applying Lemma \ref{lem: global_grad_var_bound_nonconvex} to this inequality, there exists $\eta = \Theta(1/L \wedge 1/(K\zeta) \wedge \sqrt{b}/(\sqrt{K}L) \wedge \sqrt{Pb}/(\sqrt{KT}L)$ such that 
\begin{align*}
    \frac{1}{T}\sum_{t=1}^T\frac{1}{P}\sum_{p=1}^P\mathbb{E}\|\nabla f(x_t^{(p), \mathrm{out}})\|^2 \leq \Theta\left(\frac{1}{\eta KT}\right)(\mathbb{E}f(x_{0}) - \mathbb{E}f(x_T)) + \Theta(1)\mathds{1}_{\widetilde b < \frac{1}{P}\sum_{p=1}^P\#\mathrm{supp}( D_p)}\frac{\sigma^2}{P\widetilde b}.
\end{align*}
From the definitions of $\widetilde x_s$ and $\widetilde x_s^{\mathrm{out}}$, we obtain
\begin{align*}
    \mathbb{E}\|\nabla f(\widetilde x_s^{\mathrm{out}})\|^2 \leq \Theta\left(\frac{1}{\eta KT}\right)(\mathbb{E}f(\widetilde x_{s-1}) - \mathbb{E}f(\widetilde x_{s})) + \Theta(1)\mathds{1}_{\widetilde b < \frac{1}{P}\sum_{p=1}^P\#\mathrm{supp}( D_p)}\frac{\sigma^2}{P\widetilde b}.
\end{align*}
Finally, averaging this inequality from $s=1$ to $S$ and using Assumption \ref{assump: optimal_sol} yield the desired result. \qed

\begin{corollary}
\label{cor: complexity_nonconvex}
Suppose that Assumptions \ref{assump: heterogeneous},  \ref{assump: local_loss_smoothness}, \ref{assump: optimal_sol} and \ref{assump: bounded_loss_gradient_variance} hold. We denote $n := \sum_{p=1}^P\# \mathrm{supp}(D_p)$. Let $\widetilde b = \Theta((n/P) \wedge (\sigma^2/(P\varepsilon)))$. Then, 
there exists $\eta = \Theta(1/L \wedge 1/(K\zeta) \wedge \sqrt{b}/(\sqrt{K}L) \wedge \sqrt{Pb}/(\sqrt{KT}L)))$ such that BVR-L-SGD($\widetilde x_0$, $\eta$, $b$, $\widetilde b$, $K$, $T$, $S$) with $S = \Theta(1 + 1/(\eta KT\varepsilon))$ satisfies 
\begin{align*}
    \mathbb{E}\|\nabla f(\widetilde x^{\mathrm{out}})\|^2 \leq \Theta(\varepsilon).
\end{align*}
Moreover, the total communication complexity $ST$ is 
\begin{align*}
    \Theta\left(\frac{L}{K\varepsilon} + \frac{\zeta}{\varepsilon} + \frac{L}{\sqrt{Kb}\varepsilon} + \sqrt{\frac{T}{KbP}}\frac{L}{\varepsilon} + T\right).
\end{align*}

\end{corollary}

\begin{remark}[Communication efficiency]
Given local computation budget $\mathcal B$, we set $T = \Theta(1 + \widetilde b/\mathcal B)$ and $Kb = \Theta(\mathcal B)$ with $b \leq \Theta(\sqrt{\mathcal B})$, where $\widetilde b$ was defined in Corollary \ref{cor: complexity_nonconvex}. Then, we have the averaged number of local computations per communication round $Kb + \widetilde b / T = \Theta(\mathcal B)$ and the total communication complexity with budget $\mathcal B$ becomes $\Theta((L/(\sqrt{\mathcal B}\varepsilon) + \sqrt{n\wedge (\sigma^2/\varepsilon)}L/(\mathcal BP\varepsilon) + (n\wedge (\sigma^2/\varepsilon))/(\mathcal B P) +  \zeta/\varepsilon)$. 
\end{remark}

\section{Practical Implementation of BVR-L-SGD}
\label{app_sec: implementation}
In this section, we give practical implementation details of BVR-L-SGD (Algorithm \ref{alg: bvr_l_sgd_practical}). The blue texts indicates the changes from the original algorithm (Algorithm \ref{alg: bvr_l_sgd}) for more specific, and computational and communication efficient procedures. \par
In line 1, we set $T = \lceil 1 + \widetilde b / (Kb)\rceil$, which has been theoretically determined. In line 16, we at first pick a worker $\hat p$ uniformly random and send aggregated variance reduced gradient $\widetilde v_t$ to it. Then, worker $\hat p$ runs Local-Routine using $\widetilde v_t$ (line 18). Central server receive its output and broadcast it to all the worker (line 19). Note that Algorithm \ref{alg: bvr_l_sgd_practical} only requires single aggregation and single broadcast for each $t \in [T]$.
\begin{algorithm}[t]
\caption{Practical Implementation of BVR-L-SGD($\widetilde x_0$, $\eta$, $b$, $\widetilde b$, $K$, $S$)}
\label{alg: bvr_l_sgd_practical}
\begin{algorithmic}[1]
\STATE {\color{blue}Set $T = \lceil 1 + \widetilde b / (Kb) \rceil$}. 
\FOR {$s=1$ to $S$}
    \FOR {$p=1$ to $P$ \it{in parallel}}
        \IF {$\widetilde b \geq                                 \frac{1}{P}\sum_{p=1}^P\#\mathrm{supp}(D_p)$} \label{line: local_full_grad}
        \STATE $\widetilde \nabla^{(p)} = \nabla f_p(\widetilde x_{s-1})$.
        \ELSE
        \STATE $\widetilde \nabla^{(p)} = \frac{1}{\widetilde b}\sum_{l=1}^{\widetilde b} \nabla \ell(\widetilde x_{s-1}, z_l)$ for $\widetilde b$ IID samples $z_l \sim D_p$.
        \ENDIF
    \ENDFOR
    \STATE {\color{blue}{\bf{Central Server}}:  aggregate $\{\widetilde \nabla^{(p)}\}_{p=1}^P$ and broadcast $\widetilde v_0 = \frac{1}{P}\sum_{p=1}^P \widetilde \nabla^{(p)}$ to all the workers.}
    \STATE Set $x_0 = x_{-1} = \widetilde x_{s-1}$.
    \FOR {$t=1$ to $T$}
        \FOR {$p=1$ to $P$ in parallel}
        \STATE $g_t^{(p)}(x_{t-1}) = \frac{1}{Kb}\sum_{l=1}^{Kb} \ell(x_{t-1}, z_{l})$ and $g_t^{(p)}(x_{t-2}) =  \frac{1}{Kb}\sum_{l=1}^{Kb} \ell(x_{t-2}, z_{l})$ for $z_l \overset{i.i.d.}{\sim} D_p$.
        \STATE $\widetilde v_t^{(p)} = g_t^{(p)}(x_{t-1}) - g_{t}^{(p)}(x_{t-2}) + \widetilde v_{t-1}^{(p)}$.
        \ENDFOR
        \STATE {\color{blue}Pick $\hat p \sim [P]$ uniformly at random}.
        \STATE {\color{blue}{\bf{Central Server}}: aggregate $\{\widetilde v_t^{(p)}\}_{p=1}^P$ and send $\widetilde v_t = \frac{1}{P}\sum_{p=1}^P \widetilde v_t^{(p)}$ to worker $\hat p$}.
        \STATE {\color{blue}$x_t^{(\hat p)}, \_ =$ Local-Routine($\hat p$, $x_{t-1}$, $\eta$, $\widetilde v_t$, $b$, $K$)}
        \STATE {\color{blue}{\bf{Central Server}}: receive $x_t^{(\hat p)}$ and broadcast $x_t = x_t^{(\hat p)}$ to all the workers}.
    \ENDFOR
    \STATE {\color{blue}Set $\widetilde x_s = x_T$.}
\ENDFOR
\STATE {\bf{Return:}} {\color{blue}$\widetilde x^{\mathrm{out}} = \widetilde x_{S}$.}
\end{algorithmic}
\end{algorithm}

\section{Supplementary of Numerical Experiments}\label{app_sec: experiment}
\subsection*{Parameter Tuning}
For all the implemented algorithms, the only tuning parameter was learning rate $\eta$. We ran each algorithm with $\eta \in \{0.005, 0.01, 0.05, 0.1, 0.5, 1.0\}$ and chose the one that maximized the minimum train accuracy at the last $100$ global iterates to take into account not only convergence speed but also stability of convergence.  
\subsection*{Additional Numerical Results}
Here, we provide the full results in our numerical experiments. Figures \ref{app_fig: by_rounds_b=256}, \ref{app_fig: by_rounds_b=512} and \ref{app_fig: by_rounds_b=1024} show the comparisons of train loss, train accuracy, test loss and test accuracy for various $q$ with fixed local computation budget $\mathcal B = 256, 512$ and $1,024$ respectively.

\begin{figure}[t]
\begin{subfigmatrix}{4}
\subfigure[Train Loss]{\includegraphics[width=4cm]{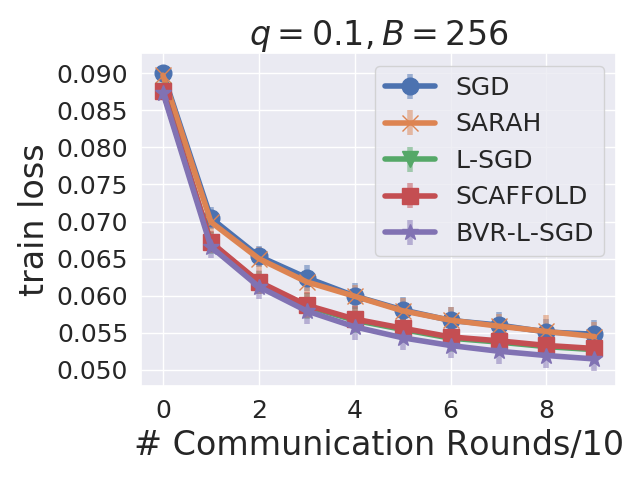}}
\subfigure[Train Accuracy]{\includegraphics[width=4cm]{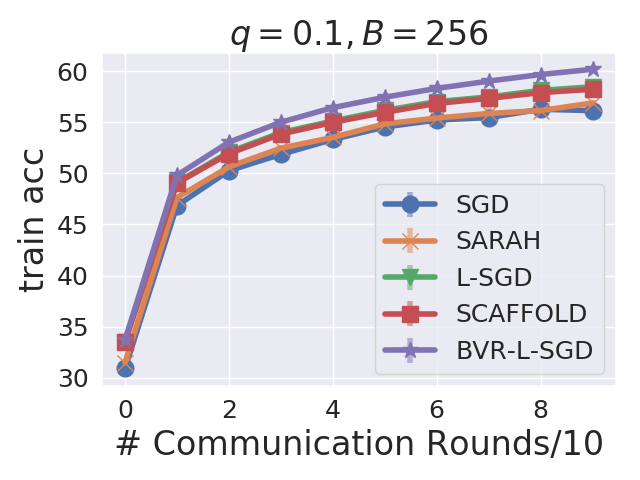}}
\subfigure[Test Loss]{\includegraphics[width=4cm]{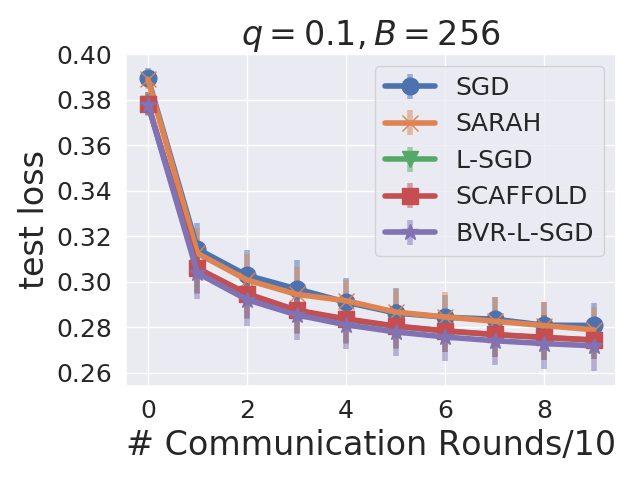}}
\subfigure[Test Accuracy]{\includegraphics[width=4cm]{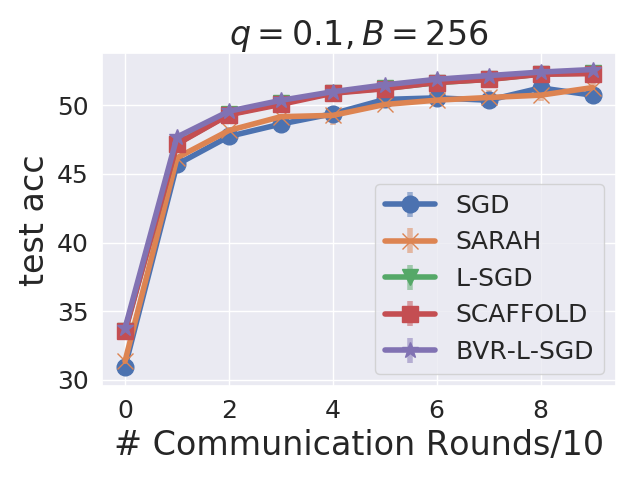}}
\end{subfigmatrix}
\begin{subfigmatrix}{4}
\subfigure[Train Loss]{\includegraphics[width=4cm]{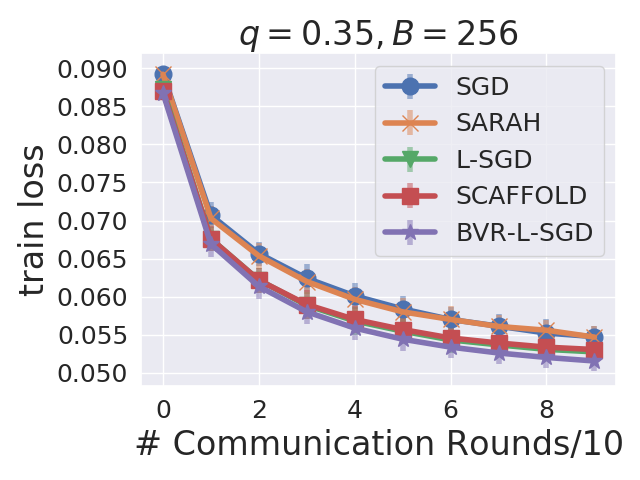}}
\subfigure[Train Accuracy]{\includegraphics[width=4cm]{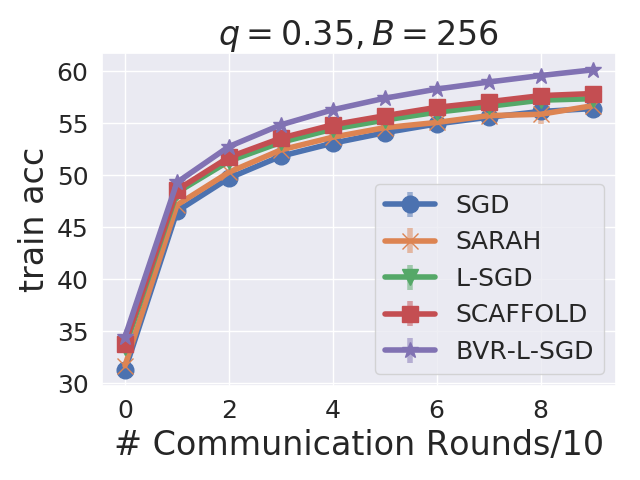}}
\subfigure[Test Loss]{\includegraphics[width=4cm]{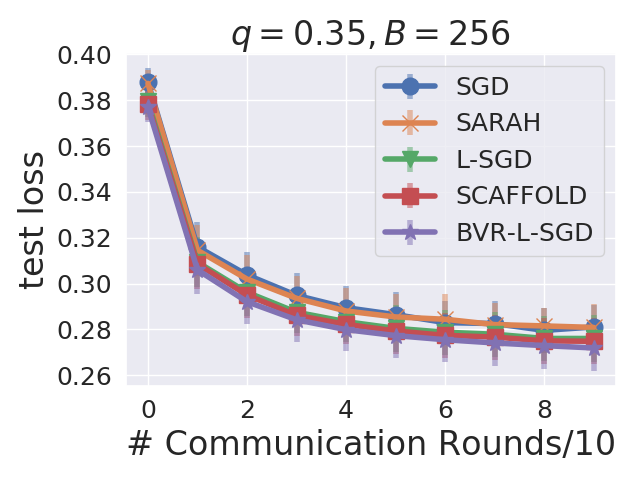}}
\subfigure[Test Accuracy]{\includegraphics[width=4cm]{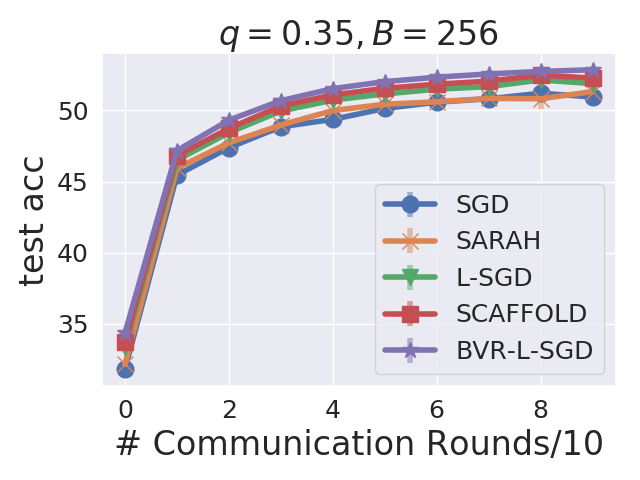}}
\end{subfigmatrix}
\begin{subfigmatrix}{4}
\subfigure[Train Loss]{\includegraphics[width=4cm]{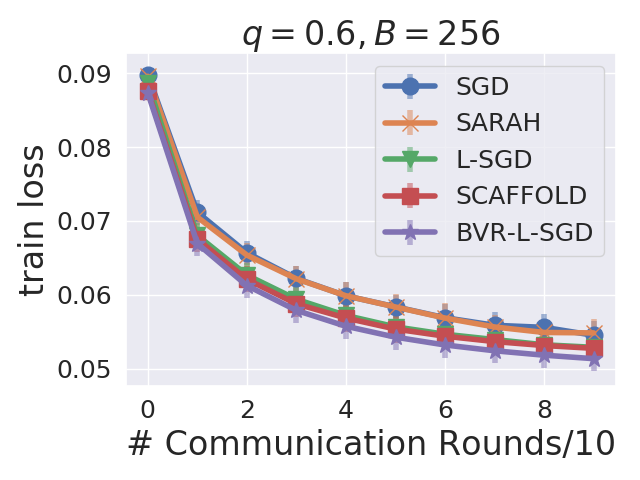}}
\subfigure[Train Accuracy]{\includegraphics[width=4cm]{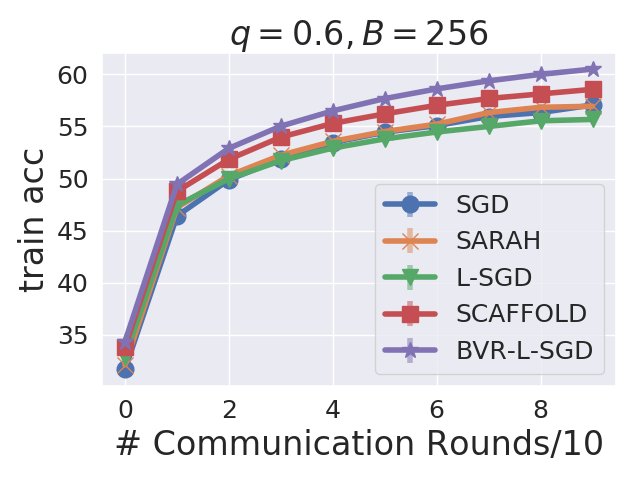}}
\subfigure[Test Loss]{\includegraphics[width=4cm]{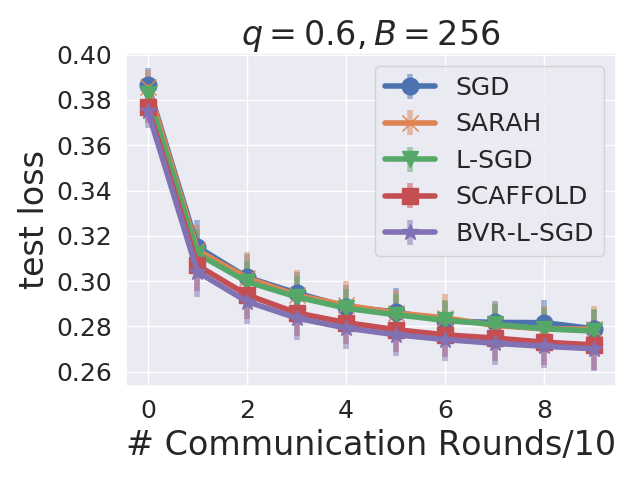}}
\subfigure[Test Accuracy]{\includegraphics[width=4cm]{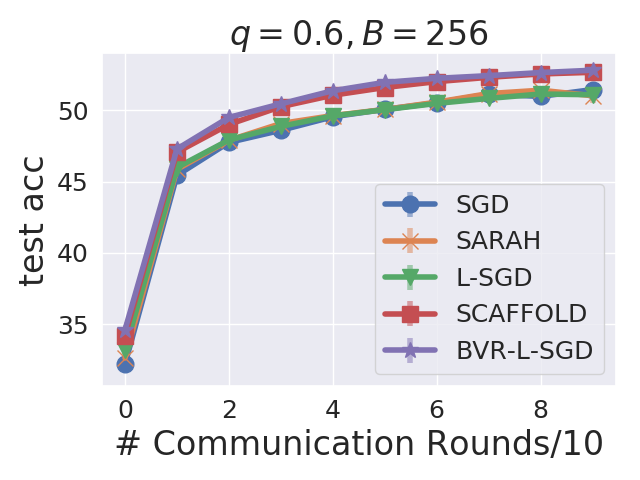}}
\end{subfigmatrix}
\begin{subfigmatrix}{4}
\subfigure[Train Loss]{\includegraphics[width=4cm]{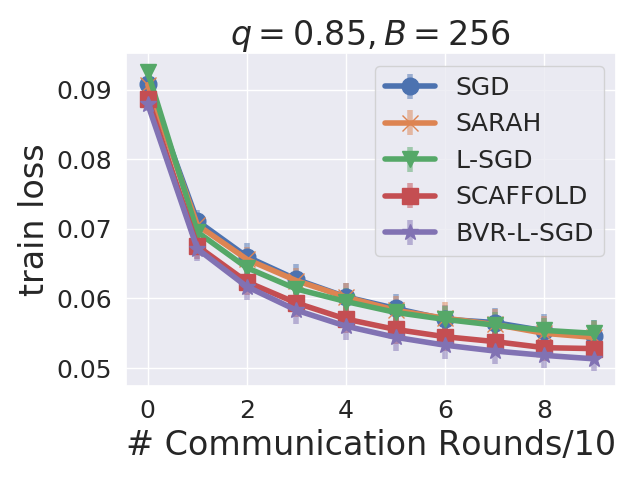}}
\subfigure[Train Accuracy]{\includegraphics[width=4cm]{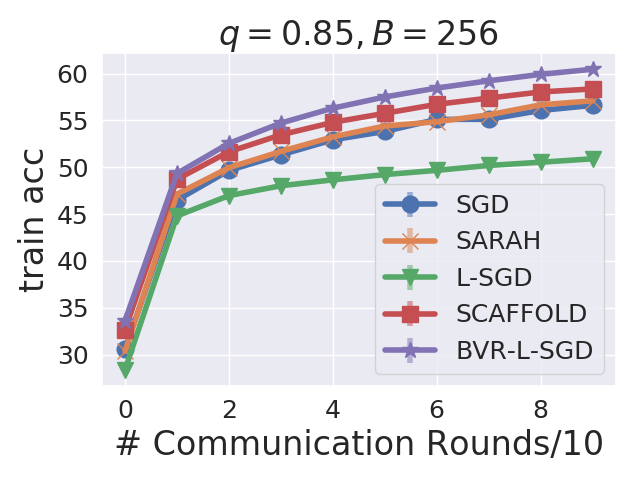}}
\subfigure[Test Loss]{\includegraphics[width=4cm]{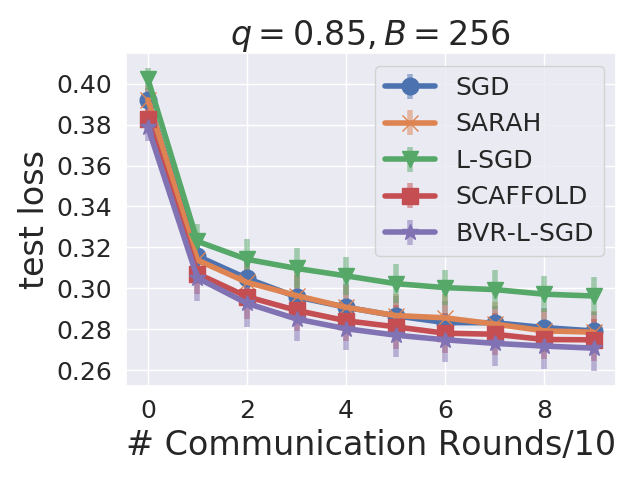}}
\subfigure[Test Accuracy]{\includegraphics[width=4cm]{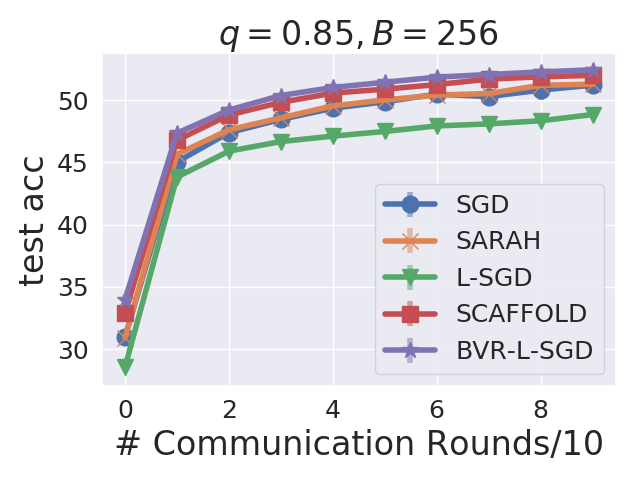}}
\end{subfigmatrix}
\caption{Comparison of the train loss and test accuracy against the number of communication rounds for local computation budget $\mathcal B = 256$. }
\label{app_fig: by_rounds_b=256}
\end{figure}

\begin{figure}[t]
\begin{subfigmatrix}{4}
\subfigure[Train Loss]{\includegraphics[width=4cm]{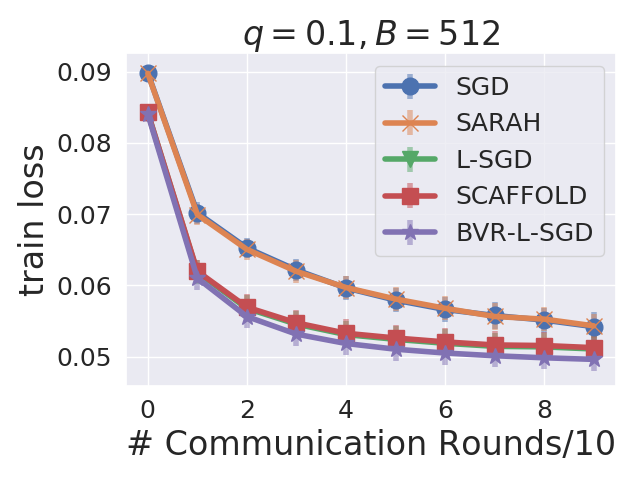}}
\subfigure[Train Accuracy]{\includegraphics[width=4cm]{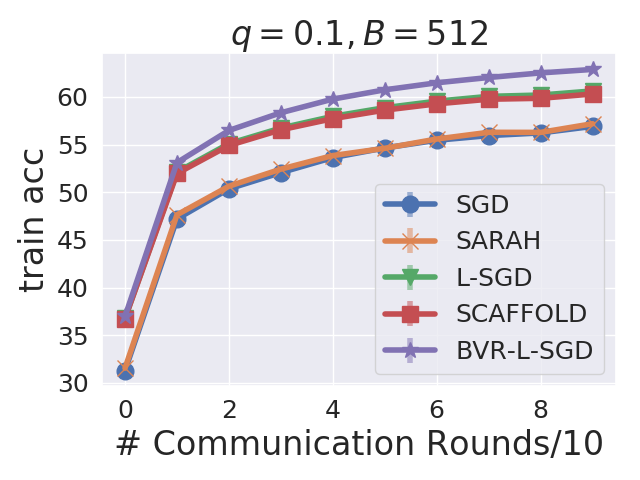}}
\subfigure[Test Loss]{\includegraphics[width=4cm]{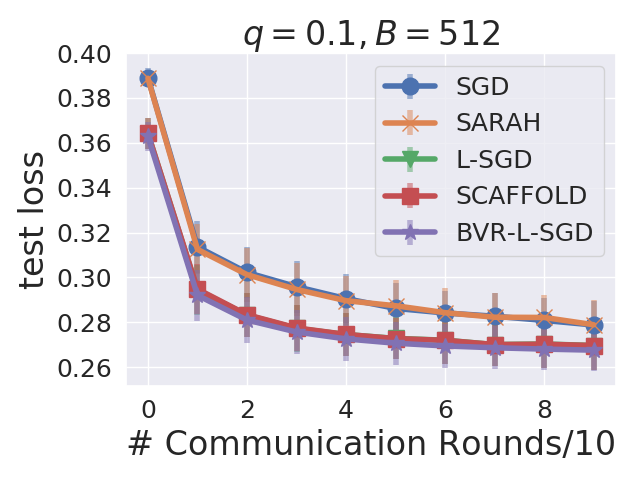}}
\subfigure[Test Accuracy]{\includegraphics[width=4cm]{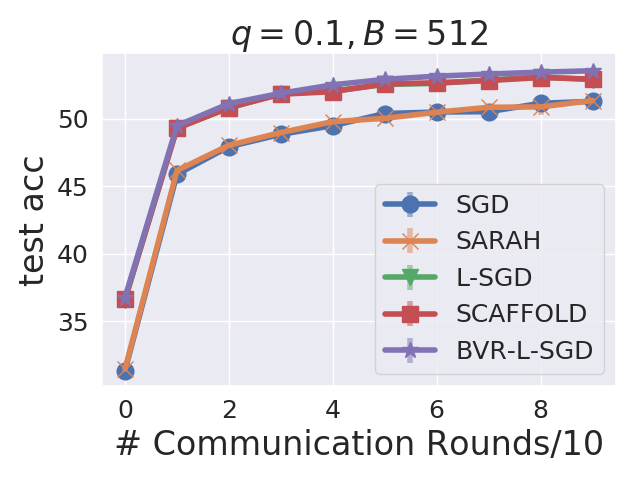}}
\end{subfigmatrix}
\begin{subfigmatrix}{4}
\subfigure[Train Loss]{\includegraphics[width=4cm]{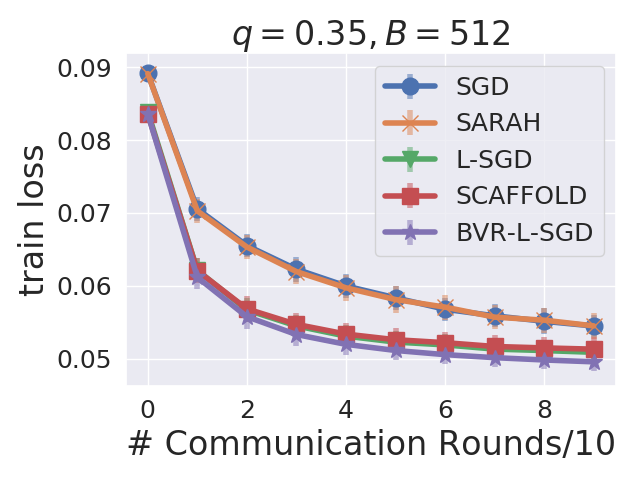}}
\subfigure[Train Accuracy]{\includegraphics[width=4cm]{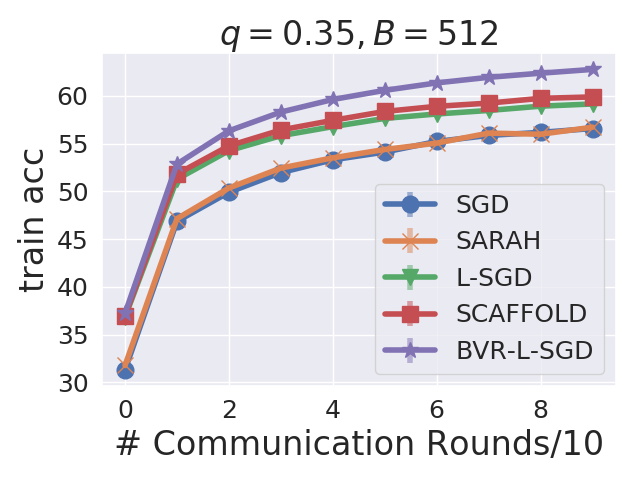}}
\subfigure[Test Loss]{\includegraphics[width=4cm]{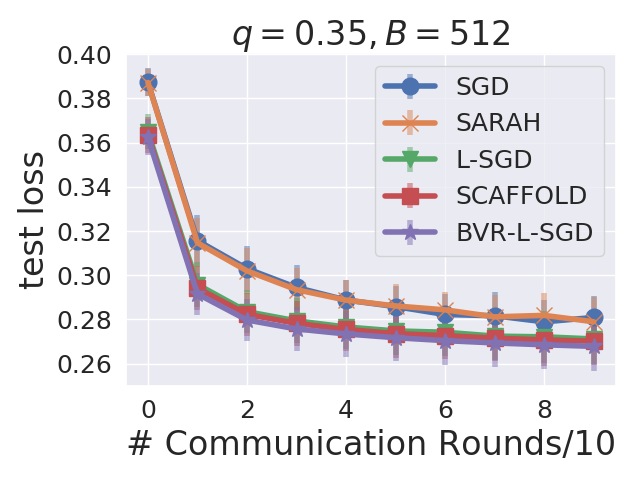}}
\subfigure[Test Accuracy]{\includegraphics[width=4cm]{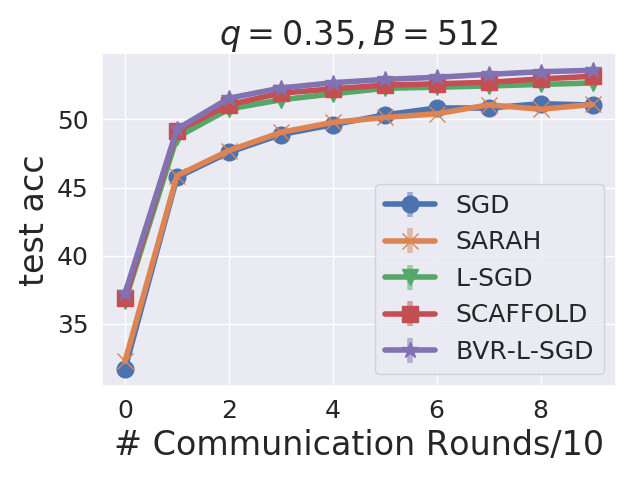}}
\end{subfigmatrix}
\begin{subfigmatrix}{4}
\subfigure[Train Loss]{\includegraphics[width=4cm]{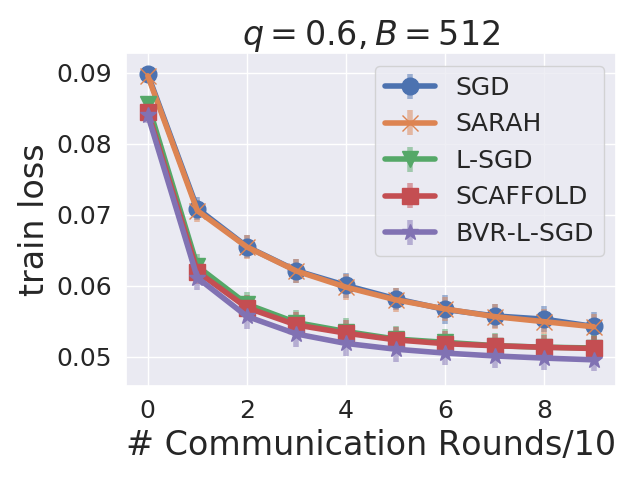}}
\subfigure[Train Accuracy]{\includegraphics[width=4cm]{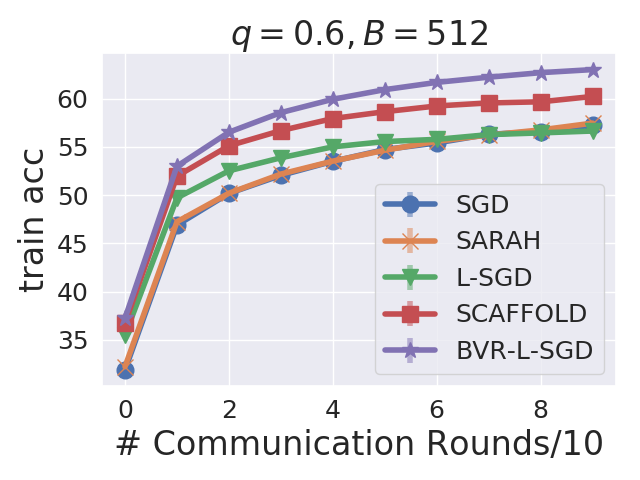}}
\subfigure[Test Loss]{\includegraphics[width=4cm]{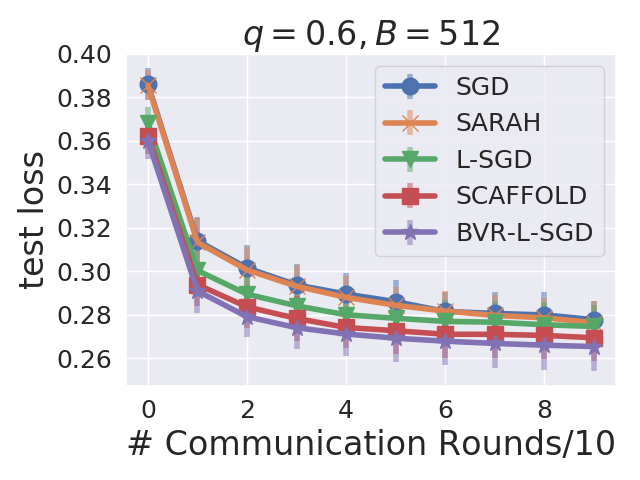}}
\subfigure[Test Accuracy]{\includegraphics[width=4cm]{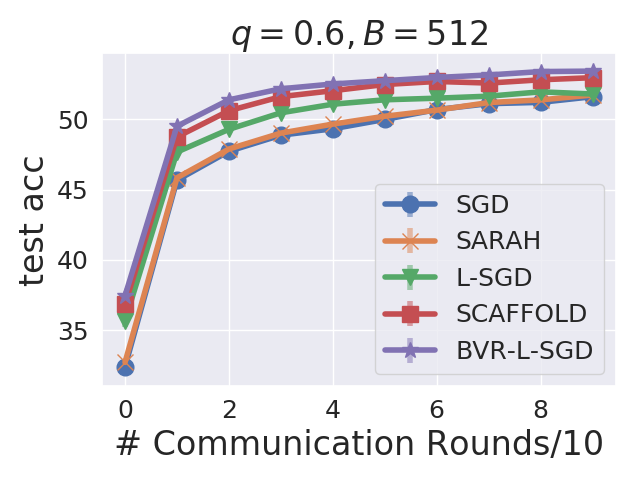}}
\end{subfigmatrix}
\begin{subfigmatrix}{4}
\subfigure[Train Loss]{\includegraphics[width=4cm]{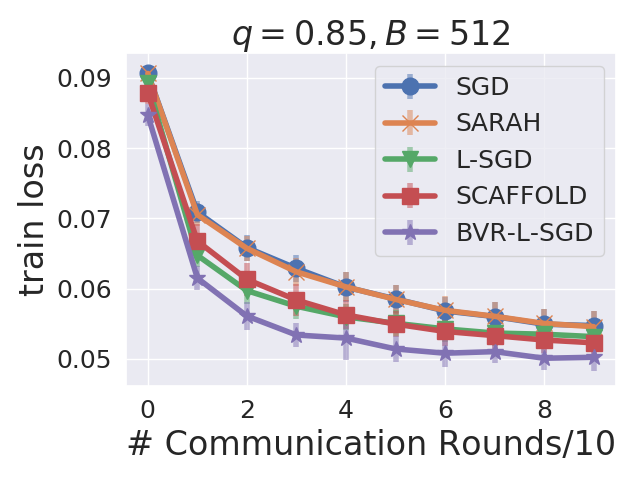}}
\subfigure[Train Accuracy]{\includegraphics[width=4cm]{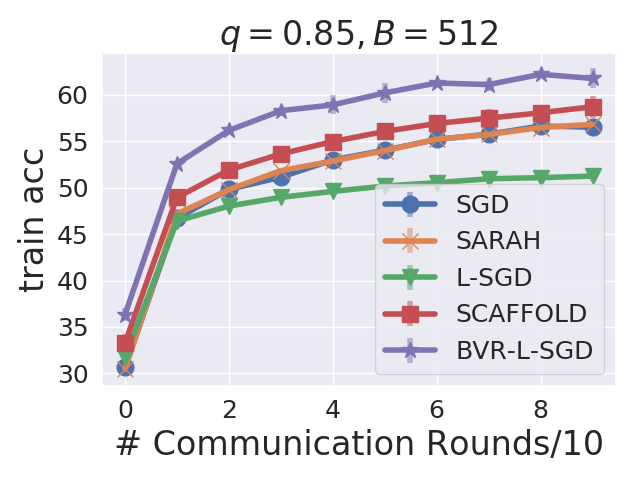}}
\subfigure[Test Loss]{\includegraphics[width=4cm]{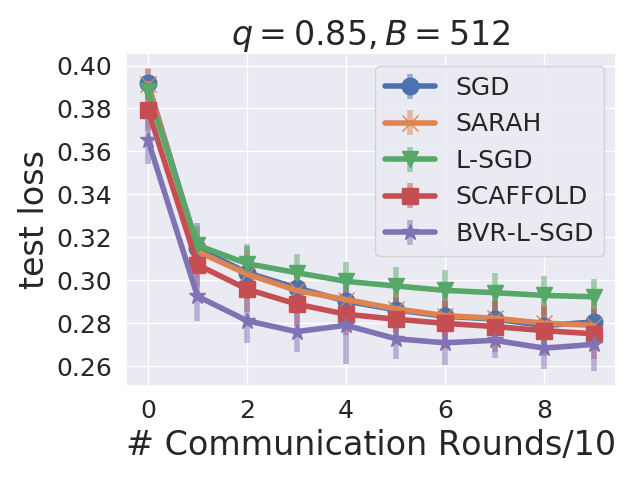}}
\subfigure[Test Accuracy]{\includegraphics[width=4cm]{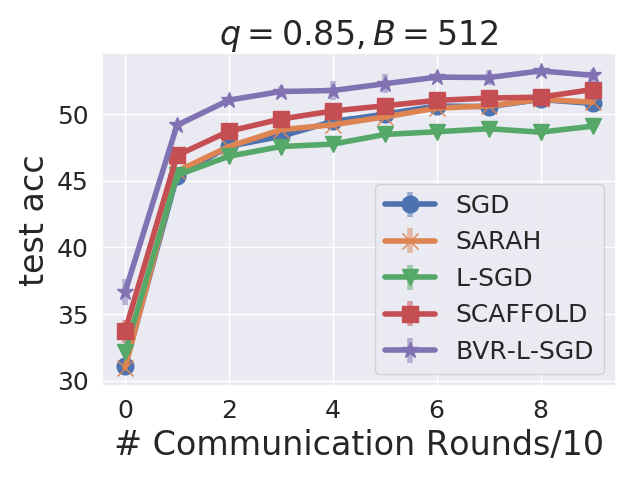}}
\end{subfigmatrix}
\caption{Comparison of the train loss and test accuracy against the number of communication rounds for local computation budget $\mathcal B = 512$. }
\label{app_fig: by_rounds_b=512}
\end{figure}

\begin{figure}[t]
\begin{subfigmatrix}{4}
\subfigure[Train Loss]{\includegraphics[width=4cm]{figs/by_rounds/q=0.1/B=1024/train_loss.png}}
\subfigure[Train Accuracy]{\includegraphics[width=4cm]{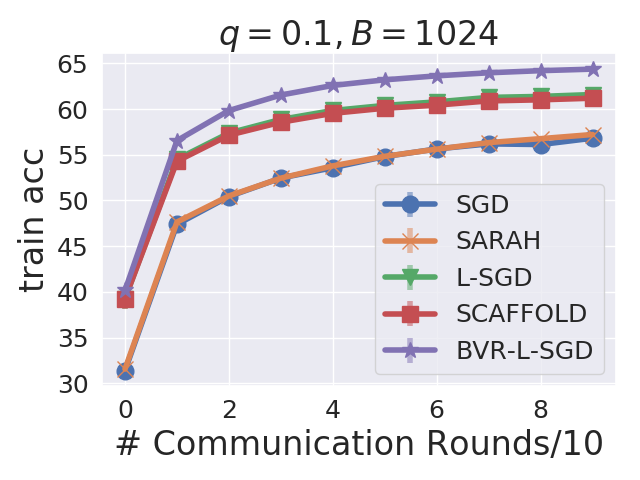}}
\subfigure[Test Loss]{\includegraphics[width=4cm]{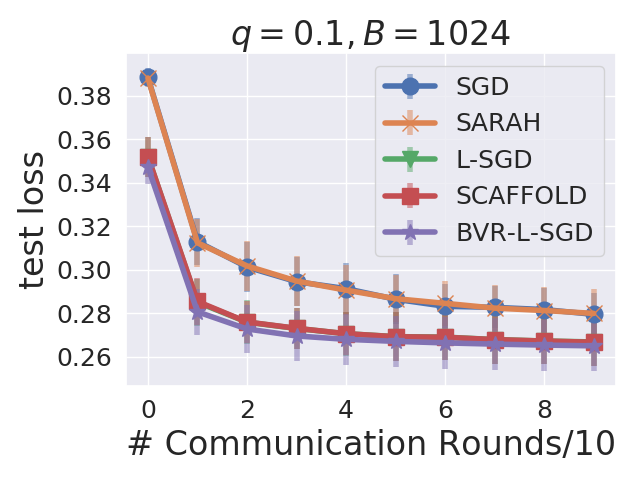}}
\subfigure[Test Accuracy]{\includegraphics[width=4cm]{figs/by_rounds/q=0.1/B=1024/test_acc.png}}
\end{subfigmatrix}
\begin{subfigmatrix}{4}
\subfigure[Train Loss]{\includegraphics[width=4cm]{figs/by_rounds/q=0.35/B=1024/train_loss.png}}
\subfigure[Train Accuracy]{\includegraphics[width=4cm]{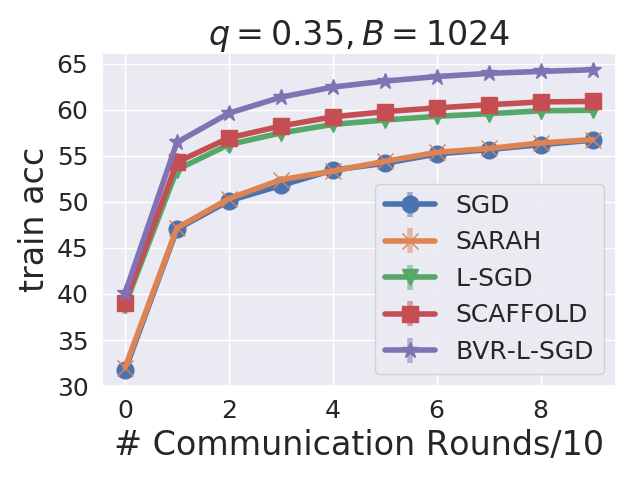}}
\subfigure[Test Loss]{\includegraphics[width=4cm]{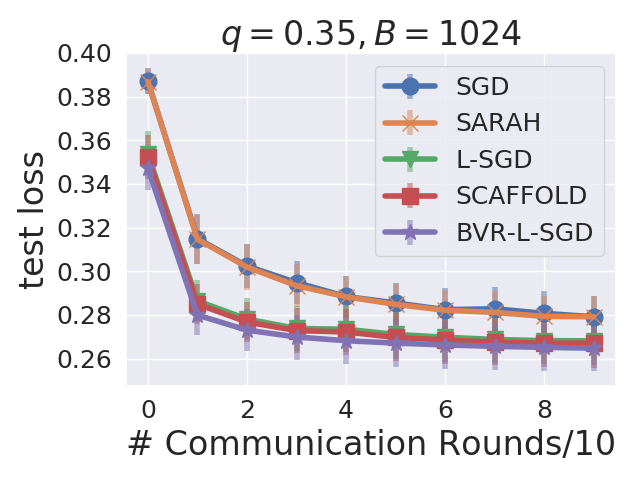}}
\subfigure[Test Accuracy]{\includegraphics[width=4cm]{figs/by_rounds/q=0.35/B=1024/test_acc.png}}
\end{subfigmatrix}
\begin{subfigmatrix}{4}
\subfigure[Train Loss]{\includegraphics[width=4cm]{figs/by_rounds/q=0.6/B=1024/train_loss.png}}
\subfigure[Train Accuracy]{\includegraphics[width=4cm]{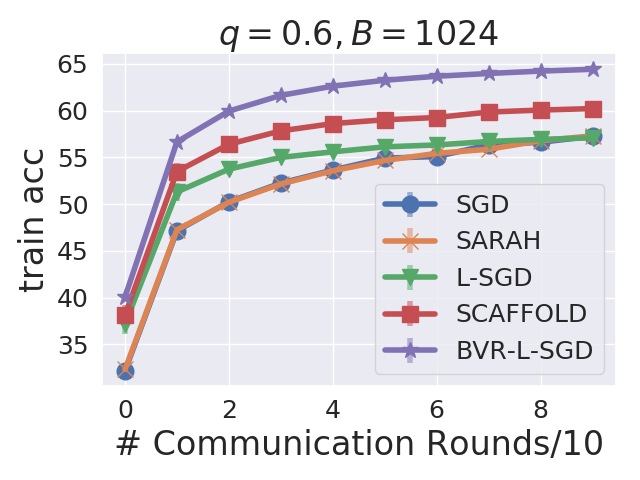}}
\subfigure[Test Loss]{\includegraphics[width=4cm]{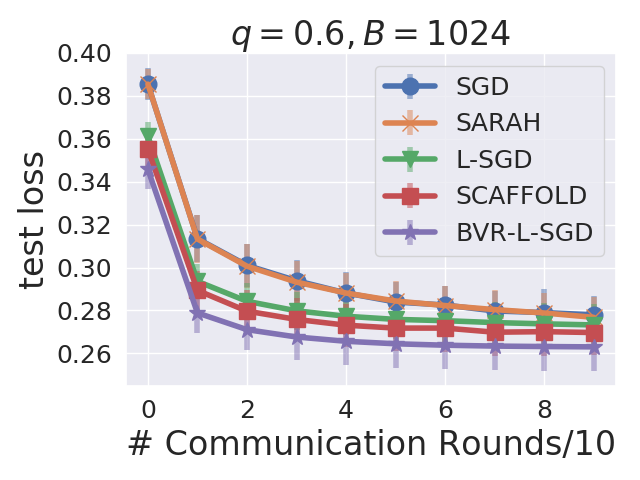}}
\subfigure[Test Accuracy]{\includegraphics[width=4cm]{figs/by_rounds/q=0.6/B=1024/test_acc.png}}
\end{subfigmatrix}
\begin{subfigmatrix}{4}
\subfigure[Train Loss]{\includegraphics[width=4cm]{figs/by_rounds/q=0.85/B=1024/train_loss.png}}
\subfigure[Train Accuracy]{\includegraphics[width=4cm]{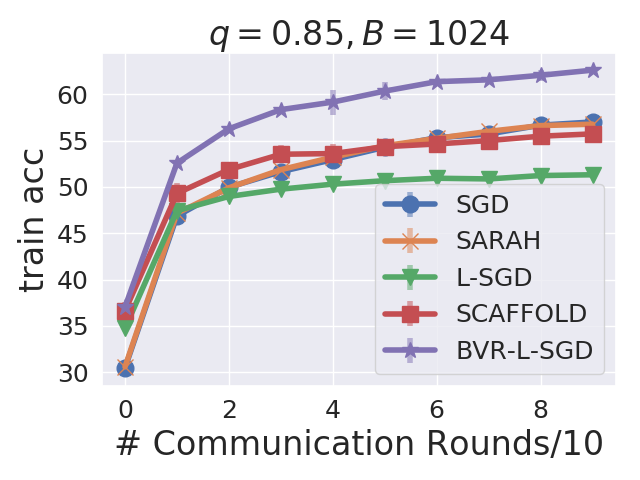}}
\subfigure[Test Loss]{\includegraphics[width=4cm]{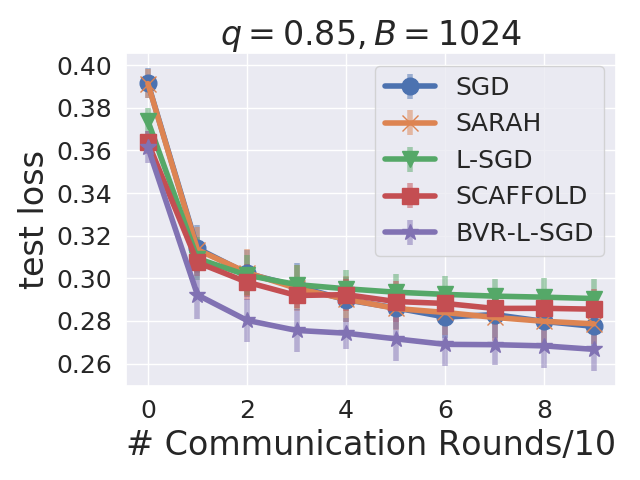}}
\subfigure[Test Accuracy]{\includegraphics[width=4cm]{figs/by_rounds/q=0.85/B=1024/test_acc.png}}
\end{subfigmatrix}
\caption{Comparison of the train loss and test accuracy against the number of communication rounds for local computation budget $\mathcal B = 1,024$. }
\label{app_fig: by_rounds_b=1024}
\end{figure}

\subsection*{Computing Infrastructures}
\begin{itemize}
    \item OS: Ubuntu 16.04.6
    \item CPU: Intel(R) Xeon(R) CPU E5-2680 v4 @ 2.40GHz
    \item CPU Memory: 128 GB.
    \item GPU: NVIDIA Tesla P100.
    \item GPU Memory: 16 GB
    \item Programming language: Python 3.7.3.
    \item Deep learning framework: Pytorch 1.3.1.
\end{itemize}

\end{document}